\documentclass[11pt]{article}
\usepackage{libertine}
\usepackage{amsfonts,amsmath,bm,amssymb,xspace,amsthm}
\usepackage{multirow,graphicx,algorithmic,rotating}
\usepackage{url,cite}
\usepackage{fullpage}
\usepackage[small,bf]{caption}
\newtheorem{theorem}{Theorem}[section]

\newtheorem{proposition}[theorem]{Proposition}
\newtheorem{definition}[theorem]{Definition}

\theoremstyle{definition}
\newtheorem{example}[subsection]{Example}

\renewenvironment{proof}{\paragraph{Proof:}}{\hfill$\square$}

\usepackage[numbers,sort,compress]{natbib}

\usepackage[utf8]{inputenc} %
\usepackage[T1]{fontenc}    %
\usepackage{hyperref}       %
\usepackage{booktabs}       %
\usepackage{amsfonts}       %
\usepackage{nicefrac}       %
\usepackage{microtype}      %

\usepackage{mathtools}
\usepackage{makecell}

\usepackage[labelfont=bf]{subcaption}
\usepackage[labelfont=bf]{caption}

\def \figwidthbase{\columnwidth}
\def \fullsize{0.75}
\def \halfsize{0.37}

\newif\ifcompileapp
\compileapptrue

\clearpage{}%

\newtheorem{result}{Result}

\usepackage{algorithm2e}

\renewcommand{\vec}{\mathbf}
\newcommand{\uf}{\vec{p}}
\newcommand{\itf}{\vec{q}}

\newcommand{\calP}{\mathcal{P}}
\newcommand{\calM}{\mathcal{M}}
\renewcommand{\r}{\vec{r}}

\newcommand{\comp}{\mathsf{c}}

\newcommand{\proj}{\Pi}

\newcommand{\nusers}{n}
\newcommand{\nitems}{m}

\newcommand{\R}{\mathbb{R}}
\newcommand{\calR}{\mathcal{R}}

\DeclareMathOperator*{\minimize}{minimize}
\DeclareMathOperator*{\maxn}{max^{(N)}}
\DeclareMathOperator*{\maxnp}{max^{(N')}}

\clearpage{}%

\begin{document}

\title{Recommendations and User Agency: The Reachability of Collaboratively-Filtered Information}

\author{Sarah Dean${}^1$, Sarah Rich${}^2$, Benjamin Recht${}^1$\\
${}^1$Department of EECS,  University of California, Berkeley\\
${}^2$Canopy Crest}

\date{December 2019}

\maketitle
\vspace{-0.3in}

\begin{abstract}

Recommender systems often rely on models which are trained to maximize accuracy in predicting user preferences.
When the systems are deployed, these models determine the availability of content and information to different users. 
The gap between these objectives gives rise to a potential for unintended consequences, contributing to phenomena such as filter bubbles and polarization.
In this work, we consider directly the information availability problem through the lens of user recourse.
Using ideas of reachability, we propose a computationally efficient audit for top-$N$ linear recommender models.
Furthermore, we describe the relationship between model complexity and the effort necessary for users to exert control over their recommendations.
We use this insight to provide a novel perspective on the user cold-start problem.
Finally, we demonstrate these concepts with an empirical investigation of a state-of-the-art model trained on a widely used movie ratings dataset. \end{abstract}

\section{Introduction}

Recommendation systems influence the way information is presented to individuals for a wide variety of domains including music, videos, dating, shopping, and advertising.
On one hand, the near-ubiquitous practice of filtering content by predicted preferences makes the digital information overload possible for individuals to navigate.
By exploiting the patterns in ratings or consumption across users, preference predictions are useful in surfacing relevant and interesting content.
On the other hand, this personalized curation is a potential mechanism for social segmentation and polarization.
The exploited patterns across users may in fact encode undesirable biases which become self-reinforcing when used in feedback to make recommendations.

Recent empirical work shows that personalization on the Internet has a limited effect on political polarization~\cite{flaxman2016filter}, and in fact it can increase the diversity of content consumed by individuals~\cite{nguyen2014exploring}.
However, these observations follow by comparison to non-personalized baselines of cable news or well known publishers.
In a digital world where all content is algorithmically sorted by default, how do we articulate the tradeoffs involved?
In the past year, YouTube has come under fire for promoting disturbing children's content and working as an engine of radicalization~\cite{zufecki,wsjyoutube,bridle}.
This comes a push on algorithm development towards reaching 1 billion hours of watchtime per day; over 70\% of views now come from the recommended videos~\cite{cnetYoutube219}.

The Youtube controversy is an illustrative example of potential pitfalls when putting large scale machine learning-based systems in feedback with people,
and highlights the importance of creating analytical tools to anticipate and prevent undesirable behavior. 
Such tools should seek to quantify the degree to which a recommender system will meet the information needs of its users or of society as a whole, where these ``information needs'' must be carefully defined to include notions like relevance, coverage, and diversity.
An important approach involves the empirical evaluation of these metrics by simulating recommendations made by models once they are trained~\cite{ekstrand2018all}.
In this work we develop a complementary approach which differs in two major ways: 
First, we directly analyze the predictive model, making it possible to understand underlying mechanisms.
Second, our evaluation considers a range of possible user behaviors rather than a static snapshot.

Drawing conclusions about the likely effects of recommendations involves treating humans as a component within the system, and the validity of these conclusions hinges on modeling human behavior.
We propose an alternative evaluation that favors the agency of individuals over the limited perspective offered by behavioral predictions.
Our main focus is on questions of \emph{possibility}: to what extent can someone be pigeonholed by their viewing history? What videos may they never see, even after a drastic change in viewing behavior? And how might a recommender system encode biases in a way that effectively limits the available library of content?

This perspective brings user agency into the center, prioritizing the  the ability for models to be as adaptable as they are accurate, able to accommodate arbitrary changes in the interests of individuals.
Studies find positive effects of allowing users to exert greater control in recommendation systems~\cite{yang2019intention,harper2015putting}.
While there are many system-level or post-hoc approaches to incorporating user feedback, we focus directly on the machine learning model that powers recommendations.

\paragraph{Contributions}
In this paper, we propose a definition of \emph{user recourse} and \emph{item availability} for recommender systems. 
This perspective extends the notion of recourse proposed by~\citet{ustun2019actionable} to multiclass classification settings and specializes to concerns most relevant for information retrieval systems.
We focus our analysis on top-$N$ recommendations made using linear predictions, a broad class including matrix factorization models.
In Section~\ref{sec:recourse_and_availability} we show how properties of latent user and item representations interact to limit or ensure recourse and availability.
This yields a novel perspective on user cold-start problems, where a user with no rating history is introduced to a system.
In Section~\ref{sec:item_results}, we propose a computationally efficient model audit.
Finally, in Section~\ref{sec:experiments}, we demonstrate how the proposed analysis can be used as a tool to interpret how learned models will interact with users when deployed.

\subsection{Related Work}

Recommendation models that incorporate user feedback for online updates have been considered from several different angles. 
The computational perspective focuses on ensuring that model updates are efficient and fast~\cite{he2016fast}. 
The statistical perspective articulates the sampling bias induced by recommendation~\cite{bonner2018causal}, while
practical perspectives identify ways to discard user interactions that are not informative for model updates~\cite{burashnikova2019sequential}.
Another body of work focuses on the learning problem, seeking to improve the predictive accuracy of models by exploiting the sequential nature of information.
This includes strategies like Thompson sampling~\cite{kawale2015efficient}, upper confidence bound approximations for contextual bandits~\cite{bouneffouf2012contextual,mary2015bandits}, and reinforcement learning~\cite{wei2017reinforcement,lei2019collaborative}. 

This body of work, and indeed much work on recommender systems, focuses on the accuracy of the model.
This encodes an implicit assumption that the primary information needs of users or society are described by predictive performance.
Alternative measures proposed in the literature include concepts related to diversity or novelty of recommendations~\cite{castells2011novelty}.
Directly incorporating these objectives into a recommender system might include further predictive models of users, e.g. to determine whether they are ``challenge averse'' or ``diversity seeking''~\cite{tintarev2017presenting}.
Further alternative criteria arise from concerns of fairness and bias, and recent work has sought to empirically quantify parity metrics on recommendations~\cite{ekstrand2018all,ekstrand2018exploring}.
In this work, we focus more directly on agency rather than predictive models or observations.

Most similar to our work is a handful of papers 
focusing on decision systems through the lens of the agency of individuals.
We are most directly inspired by the work of~\citet{ustun2019actionable} on actionable recourse for binary decisions, where users seek to change negative classification through modifications to their features.
This work has connections to concepts in explainability and transparency via the idea of \emph{counterfactual explanations}~\cite{russell2019efficient,wachter2017counterfactual}, which provide statements of the form: if a user had features $X$, then they would have been assigned alternate outcome $Y$. 
Work in strategic manipulation studies nearly the same problem with the goal of creating a decision system that is robust to malicious changes in features~\cite{hardt2016strategic,milli2018social}.

Applying these ideas to recommender systems is complex because while they can be viewed as classifiers or decision systems, there are as many outcomes as pieces of content.
Computing precise action sets for recourse for every user-item pair is unrealistic; we don't expect a user to even become aware of the majority of items.
Instead, we consider the ``reachability'' of items by users, drawing philosophically from the fields of formal verification and dynamical system analysis~\cite{bansal2017hamilton,osher2006level}. %

\section{Problem Setting}

A recommender system considers a population of users and a collection of items. We denote a ``rating'' by user $u$ of item $i$ as $r_{ui}\in\calR\subseteq \R$. This value can be either explicit (e.g. star-ratings for movies) or implicit (e.g. number of listens). 
Let $\nusers$ denote the number of users in the system and $\nitems$ denote the number of items.
Though these are both generally quite large, the number of \emph{observed ratings} is much smaller.
Let $\Omega_u$ denote the set of items whose ratings by user $u$ have been observed.
We collect these observed ratings into a sparse vector $\r_u \in \calR^\nitems$ whose values are defined at $\Omega_u$ and $0$ elsewhere.
Then a system makes recommendations with a policy $\pi(\r_u)$ which returns a subset of items.\footnote{
	While we focus on the case of deterministic policies, the analyses can be extended to randomized policies which sample from a subset of items based on their ratings. It is only necessary to define reachability with respect to probabilities of seeing an item, and then to carry through terms related to the sampling distribution.
}
We will denote the size of the returned subset as $N$, which is a parameter of the system.

We are now ready to define the reachability sub-problem for a recommender system.
We say that a user $u$ can reach item $i$ if there is some allowable modification to their rating history $\r_u$ that causes item $i$ to be recommended. 
The reachability problem for user $u$ and item $i$ is defined as
\begin{align}
\begin{split}\label{eq:general_recourse}
\minimize_{\r \in \calM(\r_u)} \quad &\mathrm{cost}(\r; \r_u)\\
\text{subject to}\quad&i\in \pi(\r)
\end{split}
\end{align}
where the modification set $\calM(\r_u)\subseteq \calR$ describes how users are allowed to modify their rating history and $\mathrm{cost}(\r; \r_u)$ describes how ``difficult'' or ``unlikely'' it is for a user to make this change. This notion of difficulty might relate discretely to the total number of changes, or to the amount that these changes deviate from the existing preferences of the user.
By defining the cost with respect to user behavior, the reachability problem encodes both the \emph{possibilities} of recommendations through its feasibility, as well as the relative \emph{likelihood} of different outcomes as modeled by the cost.

The ways that users can change their rating histories, described by the modification set $\calM(\r_u)$, depends on the design of user input to the system.
We consider a single round of user reactions to $N$ recommendations and focus on two models of user behavior: changes to existing ratings, which we will refer to as ``history edits,'' and reaction to a batch of recommended items, which we will refer to as ``reactions.''
In the first case, $\calM(\r_u)$ consists of all possible ratings on the support $\Omega_u$. In the second, it consists of all new ratings on the support $\pi(\r_u)$ combined with the existing rating history. 

The reachability problem~\eqref{eq:general_recourse} defines a quantity for each user and item in the system.
To use this problem as a metric for evaluating recommender systems, we consider both user- and item-centric perspectives. For users, this is a notion of \emph{recourse}.
\begin{definition}
The \emph{amount} of recourse available to a user $u$ is defined as the percentage of unseen items that are reachable, i.e. for which~\eqref{eq:general_recourse} is feasible. The \emph{difficulty} of recourse is defined by the average value of the recourse problem over all reachable items $i$.
\end{definition}

On the other hand, the item-centric perspective centers around notions of \emph{availability and representation}. 

\begin{definition}
The \emph{availability} of items in a recommender system is defined as the percentage of items that are reachable by some user.
\end{definition}

These definitions are important from the perspective of guaranteeing fair representation of content within recommender systems. 
This is significant for users -- for example, to what extent have their previously expressed preferences limited the content that is currently reachable?
It is equally important to content creators, for whom the ability to build an audience depends on 
the availability of their content in the recommender system overall.

In what follows, we turn our attention to specific classes of preference models, for which the reachability problem is analytically tractable. 
However, we note that estimation via sampling can provide lower bounds on availability and recourse even for black-box models,
and further explore this observation in Section~\ref{sec:sufficient_sampling}.
We keep the main focus on analysis rather than sampling because it allows us to crisply distinguish between unlikely and impossible.

\subsection{Linear Preference Models} 

While many different approaches to recommender systems exist~\cite{ricci2011introduction}, ranging from classical neighborhood models~\cite{goldberg1992using} to more recent deep neural networks~\cite{covington2016deep}, we focus our attention on \emph{linear preference models}. 
These are models which predict user rating as the dot product between user and item vectors plus bias terms: 
\[\widehat r_{ui} = \itf_i^\top \uf_u + b_i + c_u + \mu\:. \] 
We will refer to $\itf_i$ and $\uf_u$ as item and user representations.
This broad class includes both item- and user-based neighborhood models, sparse approaches like SLIM~\cite{ning2011slim}, and matrix factorization, which differ only in how the user and item representations are determined from the rating data.
For ease of exposition, we drop the bias terms for the body of the paper and focus on \emph{matrix factorization}, deferring our general results and explanation to 
\ifcompileapp
Appendix~\ref{sec:app_bias}.
\else
Appendix A.
\fi
Matrix factorization is a classical approach to recommendation which become prominent during the Netflix Prize competition~\cite{bwebb2009}.
It can be specialized to different assumptions about data and user behavior, including constrained approaches like non-negative matrix factorization~\cite{gillis2014and} or augmentations like inclusion of implicit information about preferences and additional features~\cite{paterek2007improving,hu2008collaborative,rendle2013scaling,he2016fast,koren2009bellkor,koren2009matrix}.
Due to its power and simplicity, the matrix factorization approach is still widely used; indeed it has recently been shown to be capable of attaining state-of-the-art results~\cite{rendle2019difficulty,dacrema2019we}.

In this setting, the item and user representations are referred to as factors, lying in a latent space of specified dimension $d$ which controls the complexity of the model.
The factors can be collected into matrices $P\in\R^{n\times d}$ and $Q\in\R^{m\times d}$.
Fitting the model most commonly entails solving the nonconvex minimization:
\begin{align}
\minimize_{P,Q} \sum_{u} \sum_{i\in \Omega_u} (r_{ui}-\uf_u^\top \itf_i)^2 + \Gamma(P,Q) \label{eq:MF_general_problem}
\end{align}
where $\Gamma$ is a regularizer. 

The predicted ratings of unseen items are used to make recommendations.
Specifically, we consider top-$N$ recommenders which return 
\[\{i~:~\widehat r_{ui}>\widehat r_{uj}~\text{all but at most $N$ unseen items }j\}\:.\]
For linear preference models, the condition
$\widehat r_{ui} > \widehat r_{ui}$ reduces to
\[\itf_i^\top \uf_u >\itf_j^\top \uf_u \iff (\itf_i-\itf_j)^\top \uf_u > 0 \:.\]
Thus, for fixed item vectors, a user's recommendations are determined by their representation along with a list of unseen items.
In a slight abuse of notation, we will use this fact to write the recommender policy $\pi(\uf;\Omega)$ instead of $\pi(\r)$.

When users' ratings change, their representations change as well. 
While there are a variety of possible strategies for performing online updates, we focus on the least squares approach, where
\[\uf_u = \arg\min_{\uf} \|\r_{u,\Omega_u}-Q_{\Omega_u}\uf\|_2^2 + \Gamma_u(\uf)\:.\]
This is similar to continuing an alternating least-squares (ALS) minimization of~\eqref{eq:MF_general_problem}, a common strategy~\cite{zhou2008large}.
Because we analyze a single round of recommendations, we do not consider a simultaneous updates to the item representations in $\{\itf_i\}$.

In what follows, we focus on the canonical case of $\ell_2$ regularization on user and item factors, with $\Gamma_u(\vec x) = \Gamma_i(\vec x) = \lambda\|\vec x\|_2^2$. 
With this, the user factor calculation is given by:
\begin{align}
\uf_u = (Q_{\Omega_u}^\top Q_{\Omega_u} + \lambda I)^{-1}Q^\top \r_u\:, \label{eq:uf_update}
\end{align}
which is linear in the rating vector.
In 
\ifcompileapp
Appendix~\ref{sec:app_bias},
\else
Appendix A,
\fi
we show that several other linear preference models have similar user updates. 

\section{Recourse and Availability}\label{sec:recourse_and_availability}

We begin by reformulating the reachability problem to the case of recommendations made by matrix factorization models (the general reformulation for linear models is presented in 
\ifcompileapp
Appendix~\ref{sec:app_bias}).
\else
Appendix A).
\fi
We focus this initial exposition on the simplifying case that $N=1$ and make direct connections between between model factors and the recourse and availability provided by the recommender system.

First, we consider what needs to be true for an item $i$ to be recommended.
For  for top-$1$, the constraint $i\in\pi(\uf, \Omega)$ is equivalent to requiring that
\[(\itf_i-\itf_j)^\top \uf>0~~\forall~j\notin\Omega\iff G_i\uf > 0\:. \]
where we define $G_i$ to be a $m-|\Omega|\times d$ matrix with rows given by $(\itf_i-\itf_j)$ for $j\notin\Omega$. 
This is a linear constraint on the user factor $\uf$, and the set of user factors which satisfy this constraint make up an open convex polytopic cone.
We refer to this set as the \emph{item-region} for item $i$, since any user whose latent representation falls within this region will be recommended item $i$.
The top-$1$ regions partition the latent space, as illustrated by Figure~\ref{fig:polytope_illustration} for a toy example with latent dimension $d=2$.

\begin{figure}
\centering
\begin{subfigure}[t]{\halfsize\figwidthbase}
    \centering\includegraphics[width=\fullsize\figwidthbase]{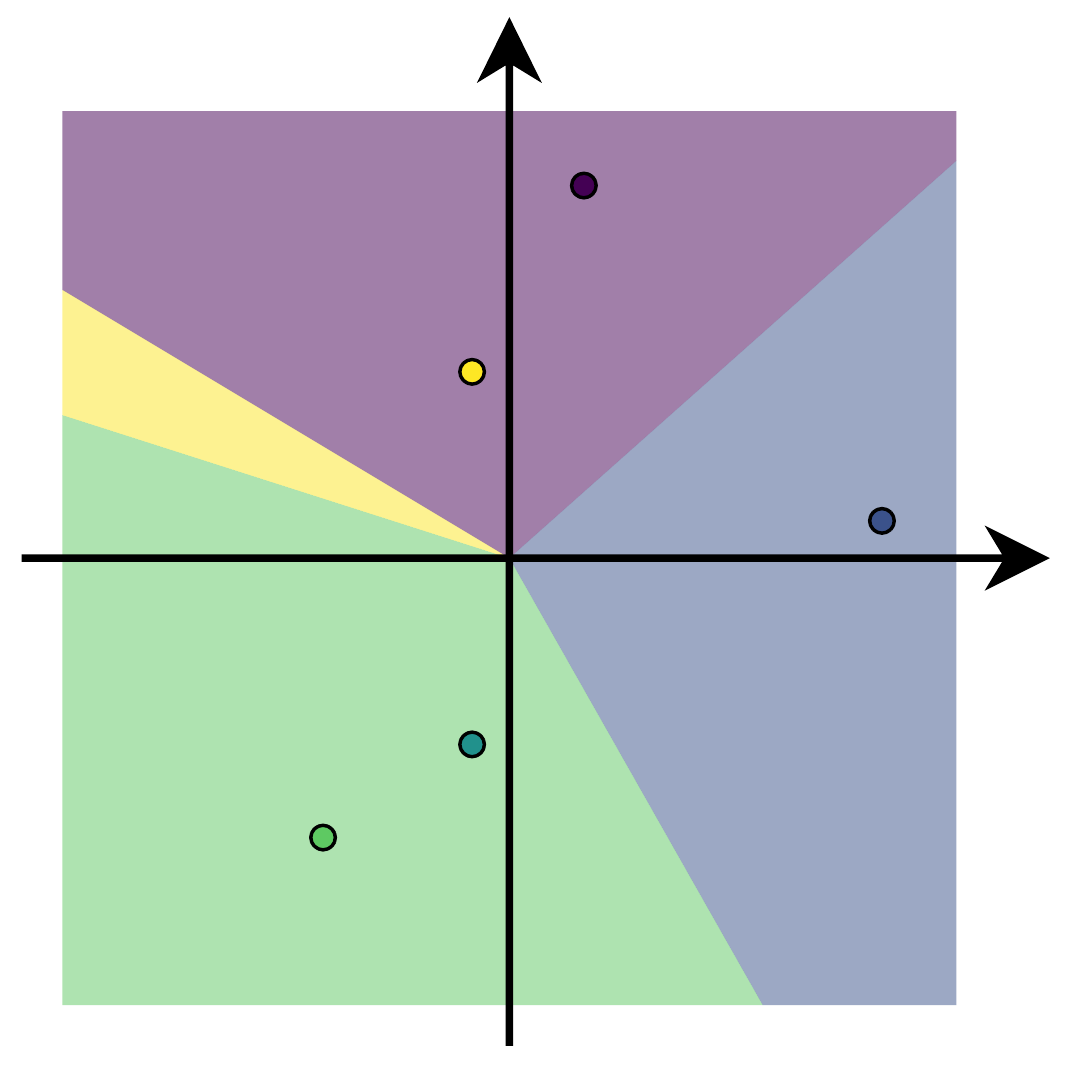}
    \caption{Items in Latent Space}
  \end{subfigure}
  \begin{subfigure}[t]{\halfsize\figwidthbase}
    \centering\includegraphics[width=\fullsize\figwidthbase]{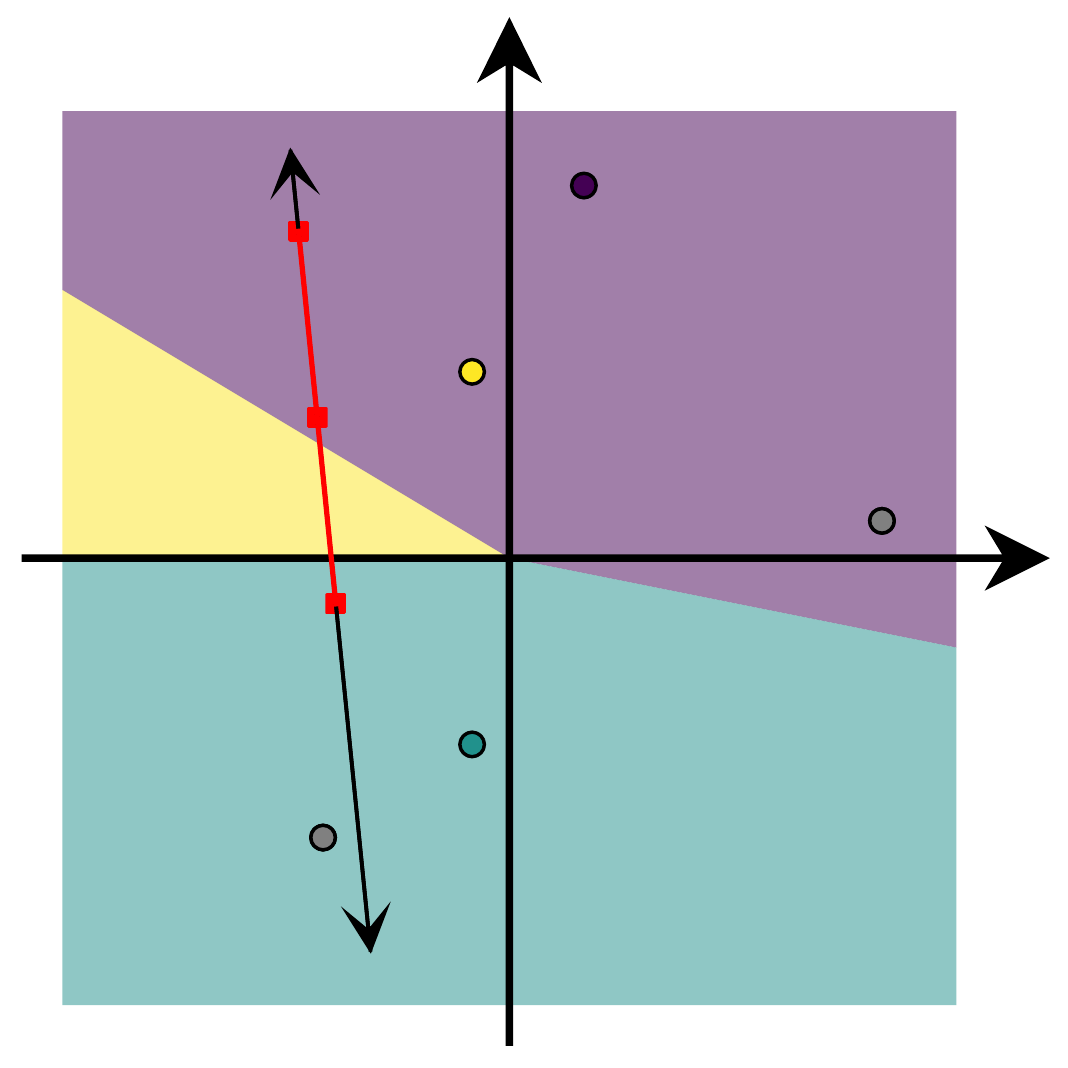}
    \caption{Availability of Items to a User}
  \end{subfigure}
\caption{
An example of item factors (indicted by colored points) in $d=2$. In (a), the top-1 regions are indicated by shaded colors.
The teal item is unavailable, and though the yellow item is reachable, it is not aligned-reachable.
In (b), the availability of items for a user who has seen the blue and the green items (now in grey) with the blue item's rating fixed.
The black line indicates how the user's representation can change
depending on their rating of the green item. The red region indicates the constraining effect of requiring bounded and integer-valued ratings, which affect the reachability of the yellow region.
}\label{fig:polytope_illustration}
\end{figure}

If item factors define regions within the latent space, user factors are points that may move between regions. %
The constraints on user actions are described by the modification set $\calM(\r_u)$. 
We will distinguish between \emph{mutable} and \emph{immutable} ratings of items within a rating vector $\r_u$.
Let $\Omega_0$ denote the set of items with immutable ratings and let $\r_0 \in \calR^{|\Omega_0|}$ denote the corresponding ratings.
Then let $\Omega_m$ denote the set of items with mutable ratings.
In what follows, we will write the full set of observed ratings $\Omega=\Omega_0\cup\Omega_m$.
Then the modification set is all rating vectors $\r\in\calR$ with:
\begin{enumerate}
  \item fixed immutable ratings, $\r_{\Omega_0} = \r_0$
  \item mutable ratings $\r_{\Omega_m}=\vec a$ for some value $\vec a \in \calR^{|\Omega_m|} $,  
  \item unseen items with no rating, $\r_{\Omega^\comp} = 0$
\end{enumerate}
The variable $\vec a$ is the decision variable in the reachability problem.

Then a user's latent factor can change as %
\begin{align*}
\uf &= (Q_\Omega^\top Q_\Omega + \lambda I)^{-1}(Q_{\Omega_0}^\top \r_0 + Q_{\Omega_m}^\top \vec a)=  \vec v_0 + B \vec a
\end{align*}
where we define
\begin{align*}
W &= (Q_\Omega^\top Q_\Omega + \lambda I)^{-1}\:,\quad
B = W Q_{\Omega_m}^\top\:,\quad \vec v_0 = W Q_{\Omega_0}^\top \r_0\:.
\end{align*}
It is thus clear that this latent factor lies in an affine subspace.
This space is anchored at $\vec v_0$ by the immutable ratings, while the mutable ratings determine the directions of possible movement.
This idea is illustrated in Figure~\ref{fig:polytope_illustration}b, which further demonstrates the limitations due to bounded or discrete ratings, as encoded in the rating set $\calR$.

We are now able to specialize reachability problem~\eqref{eq:general_recourse} for matrix factorization models:
\begin{align}
\begin{split}\label{eq:mf_recourse}
\minimize_{\vec a \in \calR^{|\Omega_m|}} \quad &\mathrm{cost}([\r_0; \vec a]; \r_u)\\
\text{subject to}\quad& G_i (\vec v_0 + B\vec a) > 0 
\end{split}
\end{align}
If the cost is a convex function and $\calR$ is a convex set, this is a convex optimization problem which can be solved efficiently.
If $\calR$ is a discrete set or if the cost function incorporates nonconvex phenomena like sparsity, then this problem can be formulated as a mixed-integer program (MIP).
Despite bad worst-case complexity, MIP can generally be solved quickly with modern software~\cite{gurobi,ustun2019actionable}.

\subsection{Item Availability}

Beyond defining the reachability problem, we seek to derive properties of recommender systems based on their underlying preference models. 
We begin by considering the feasibility of~\eqref{eq:mf_recourse} with respect to its linear inequality constraints. 
For now, we focus on the item-regions, ignoring the effects of user history $\Omega$, anchor point $\vec v_0$, and control matrix $B$.
In the following result, we consider the \emph{convex hull} of unseen item factors, which is the the smallest convex set that contains the item factors.
This can be formally written as,
\[\mathrm{conv}(\{\itf_j\}_{j})= 
\Big\{\sum_{j} \lambda_j \itf_j~:~\sum_{j} \lambda_j \leq 1~\text{and}~ \lambda \geq 0\Big\} \:.\]
Furthermore, we will consider \emph{vertices} of the convex hull, which are item factors that are not contained in the convex hull of the other factors, i.e. $\itf_i \notin \mathrm{conv}(\{\itf_j\}_{j\neq i})$.

\begin{result}\label{res:convex_hull}
In a top-$1$ recommender system, the available items are those whose factors are vertices on the convex hull of all item factors.
\end{result}
As a result, the availability of items in a top-$1$ recommender system is determined by the way the item factors are distributed in space: it is simply the percentage of item factors that are vertices of their convex hull.
The proof is provided in 
\ifcompileapp
Appendix~\ref{sec:app_bias},
\else
Appendix A,
\fi
along with proofs of all results to follow.

We can further understand the effect of limited user movement in the case that ratings are real-valued, i.e. $\calR=\R$. 
In this case, we consider
the multiplication of item factors by the transpose of the control matrix, the \emph{multiplied factors} $B^\top \itf_i$.

\begin{result} \label{res:user_convex_hull}
In a top-$1$ recommender system, a user can reach any item $i$ whose \emph{multiplied factor} is a vertex of the convex hull of all unseen multiplied item factors.
Furthermore, if the factors of the items with mutable ratings are full rank, i.e. $Q_{\Omega_m}$ has rank equal to the latent dimension of the model $d$, then item availability  implies user recourse.
\end{result}

The second statement in this result means that for a model with $100\%$ item availability, having as many mutable ratings as latent dimensions is sufficient for ensuring that users have full recourse (so long as the associated item factors are linearly independent).
This observation highlights that increased model complexity calls for more comprehensive user controls to maintain the same level of recourse.

Of course, this conclusion follows only from considering the \emph{possibilities} of user action -- to consider a notion of \emph{likelihood} for various outcomes we need to consider the cost.

\subsection{Bound on Difficulty of Recourse}

We now propose a simple model for the cost of user actions, and use this to show a bound on the difficulty of recourse for users. 
The cost of user actions can be modeled as a penalty on \emph{change from existing ratings}. 
For items whose ratings have not already been observed, we penalize instead the \emph{change from predicted ratings}.
For simplicity, we will represent this penalty as the norm of the difference.

For history edits, all mutable items have been observed, so we simply have
\[\mathrm{cost}_\mathrm{hist}(\r;\r_u) = \|\r-\r_u\|\:.\]
Additionally, all existing ratings are mutable so mutable set $\Omega_m = \Omega_u$ and immutable set $\Omega_0 = \emptyset$. 
For reactions, the ratings for the new recommended items have not been observed, so
\[\mathrm{cost}_\mathrm{react}(\r;\r_u) = \|\r_{\pi(\r_u)}-\widehat\r_{\pi(\r_u)}\|\:.\]
Additionally, the rating history is immutable so $\Omega_0 = \Omega_u$ while the mutable ratings are the recommendations with $\Omega_m = \pi(\r_u)$.

We note briefly that our choice of handling the cost of unobserved ratings as ``change from predicted ratings'' assumes a level of model validity. 
While it does allow us to avoid any external behavioral modeling, it represents a simple case that perhaps over-emphasizes the role of the model.
Exploring alternative cost functions is important for future work.

Under this model, we provide an upper bound on the difficulty of recourse.
This result holds for the case that ratings are real-valued, i.e. $\calR=\R$ and that the reachable items satisfy an \emph{alignment} condition
\ifcompileapp
(defined in~\eqref{eq:assumption} of Appendix~\ref{sec:app_bias}).
\else
(defined in (8) of Appendix A).
\fi

\begin{result} \label{res:cost_bound}
Let $\uf_u$ indicate the user's latent factor as in~\eqref{eq:uf_update} before any actions are taken or the next set of recommendations are added to the user history.
Then both in the case of full history edits and reactions,
\[\text{difficulty of recourse for user $u$} \leq \|B^\dagger\| \cdot \frac{1}{|\Omega_r|}\sum_{i\in\Omega_{r}} \|\itf_i-\uf_u\| \:,\]
where $\Omega_{r}\subseteq\Omega^\comp$ is the set of reachable items.
\end{result}
This bound depends how far item factors are from the initial latent representation of the user.
When latent representations are close together, recourse is easier or more likely--an intuitive relationship.
This quantity will be large in situations where a user is in an isolated niche, far from most of the items in latent space.
The bound also depends on the conditioning of the user control matrix $B$, which is related to the similarity between mutable items:
the right hand side of the bound will be larger for sets of mutable items that are more similar to each other.

The proof of this result hinges on showing the existence of a specific feasible point to the optimization problem in~\eqref{eq:mf_recourse}.
In the Section~\ref{sec:item_results} we will further explore this idea of feasibility to develop lower bounds on availability and recourse when $N>1$.
First, we described how the presented results can be used to evaluate solutions to the user cold-start problem.

\subsection{User Cold-Start}

The amount and difficulty of recourse for a user yields a novel perspective on how to incorporate new users into a recommender system.
The user cold-start problem is the challenge of selecting items to show a user who enters a system with no rating history from which to predict their preferences.
This is a major issue with collaboratively filtered recommendations, and systems often rely on incorporating extraneous information~\cite{schein2002methods}.
These strategies focus on presenting items which are most likely to be rated highly or to be most informative about user preferences~\cite{biswas2017combating}.

The idea of recourse offers an alternative point of view. 
Rather than evaluating a potential ``onboarding set'' only for its contribution to model accuracy, we can choose a set which additionally ensures some amount of recourse.
Looking to Result~\ref{res:user_convex_hull}, we can evaluate an onboarding set by the geometry of the \emph{multplied factors} in latent space.
In the case of onboarding, $\vec v_0=0$ and $B=WQ_{\Omega}^\top$, so the recourse evaluation involves considering the vertices of the convex hull of the columns of the matrix
$\{Q_{\Omega} W Q^\top_{\Omega^\comp} \}$.

An additional perspective is offered by considering the difficulty of recourse.
In this case, we focus on $\|B^\dagger\|$. 
If we consider an $\ell_2$ norm, then it reduces to
\begin{align*}
\|B^\dagger\| = \max_{i} \frac{\sigma_i^2+\lambda}{\sigma_i} 
\end{align*}
where $\sigma_1 \geq \sigma_2\geq ... \geq \sigma_r > 0$ are the nonzero singular values of $Q_\Omega$.
Minimizing this quantity is hard~\cite{ccivril2009selecting}, 
though the hardness of selecting informative item sets is unsurprising as it has been discussed in related settings~\cite{biswas2017combating}.
Due to their computational challenges,
we primarily propose that these metrics be used to distinguish between candidate onboarding sets, based on the ways these sets provide user control.
We leave to future work the task of generating candidate sets based on these recourse properties.

\section{Sufficient Conditions for Top-N} \label{sec:item_results}

In the previous section, we developed a characterization of reachability for top-$1$ recommender systems. 
However, most real world applications involve serving several items at once.
Furthermore, using $N>1$ can approximate the availability of items to a user over time, as they see more items and increase the size of the set $\Omega$ which is excluded from the selection. 
In this section, we focus on sufficient conditions to develop a computationally efficient model audit that provides lower bounds on the \emph{availability} of items in a model.
We further provide approximations for computing a lower bound on the \emph{recourse} available to users.

We can define an item-region for the top-$N$ case, when
$i\in\pi(\uf;\Omega)$ for any user factors in the set
\[\calP_i =\{\uf~:~(\itf_i-\itf_j)^\top \uf>0~\text{all but at most $N$ items }j\notin\Omega\}\:.\]

As in the previous section, this region is contained within the latent space, which is generally of relatively small dimension.
However, its description depends on the number of items, which will generally be quite large.
In the case of $N=1$, this dependence in linear and therefore manageable.  
For $N>1$, the item region is the union over 
polytopic cones for subsets describing ``all but at most $N$ items.''
Therefore, the description of each item region requires $\mathcal{O}(m^N)$ linear inequalities.
For systems with tens of thousands of items, even considering $N=5$ becomes prohibitively expensive.

To ease the notational burden of discussing the ranking logic around top-$N$ selection in what follows, we define the operator $\maxn$, which selects the $N$th largest value from a set. This allows us to write, for example,
\[\widehat r_{ui}>\widehat r_{uj}~\text{all but at most $N$ items }j\iff \widehat r_{ui} > {\maxn_{j\neq i}}~ \widehat r_{uj} \:.\]

\subsection{Sufficient Condition for Availability}\label{sec:sufficient_sampling}
To bypass the computational concerns, we focus on a sufficient condition for item availability.
The full description of the region $\calP_i$ is not necessary to verify non-emptiness; rather, 
showing the existence of any point in the latent space $\vec v\in\R^d$ that satisfies $\vec v \in \calP_i$ is sufficient.
Using this insight, we propose a sampling approach to determining the availability of an item.
For a fixed $\vec v$ and any $N$, it is necessary only to compute and sort $Q_{\Omega^\comp}\vec v$, which is an operation of complexity $\mathcal{O}(m^2 d\log(m))$.

While this sampling approach could make use of gridding, randomness, or empirical user factor distributions, we propose choosing the sample point $\vec v =\itf_i$.
 \begin{result}\label{res:item_aligned}
The item-region $\calP_i$ is nonempty if
\begin{align}
\delta_i = \|\itf_i\|_2^2 - {\maxn_{j\notin \Omega\cup\{i\}}}~ \itf_j^\top \itf_i > 0\:. \label{eq:aligned_reach_condition}
\end{align} 
When this condition holds, we say that item $i$ is \emph{aligned-reachable}.
The percent of items that are aligned-reachable is a lower bound on the availability of items.
\end{result}
Note that this is a sufficient rather than a necessary condition; it is possible to have $\itf_i \notin \calP_i$ for a nonempty $\calP_i$. Figure~\ref{fig:polytope_illustration} illustrates an example where this is the case: the yellow latent factor lies within $\calP_\text{purple}$ even though $\calP_\text{yellow}$ is nonempty.
As a result, aligned-reachability yields an underestimate of the availability of items in a system.

\subsection{Model Audit} \label{sec:item-audit}

To use the aligned-reachable condition~\eqref{eq:aligned_reach_condition} as a generic model audit, we need to sidestep the specificity of the set of seen items $\Omega$.
We propose an audit based on this condition with $\Omega=\emptyset$ and an increased value for $N$,
where increasing $N$ compensates for discarding the effect of the items which have been seen. 
This audit is described in Algorithm~\ref{alg:model_audit}
.
If we set $N'=N+N_h$ where $N$ is the number of items recommended by the system, then item availability has the following interpretation:
if an item is not top-$N'$ available, then that item will never be recommended to a user who has seen fewer than $N_h$ items.

\RestyleAlgo{boxruled}
\LinesNumbered
\begin{algorithm}
\SetAlgoLined
\KwResult{Lower bound on item availability $1-\frac{m_\text{unavailable}}{m}$}
initialize $m_\text{unavailable}=0$, $N'=N+N_h$\;
 \For{$i\in\{1,\dots,m\}$}{
  compute $\delta_i=\|\itf_i\|_2^2 - {\maxnp_{j\neq i}}~ \itf_j^\top \itf_i$\;
  \If{$\delta_i\leq0$}{
   $m_\text{unavailable} += 1$
 }}
 \caption{Item-Based Model Audit}\label{alg:model_audit}
\end{algorithm}

If we consider the set of all possible users to be users with a history of at most $N_h$, this model audit counts the number of aligned-unreachable items, returning lower bound on the overall \emph{availability} of items.
We can further use this model audit to propose constraints or penalties on the model during training.
Ensuring aligned-reachability is equivalent to imposing linear constraints on the matrix $A=QQ^\top$,
\[A_{ii} \geq \maxn_{j\neq i} A_{ij}\:. \]
While this constraint is not convex, relaxed versions of it could be incorporated into the optimization problem~\eqref{eq:MF_general_problem} to ensure reachability during training. 
We point this out as a potential avenue for future work.

\subsection{Sufficient Condition for Recourse} \label{sec:user_results}

User recourse inherits the computational problems described above for $N>1$.
We note that the region $\calP_i$ is not necessarily convex, though it is the union of convex regions. 
While the problem could be solved by first minimizing within each region and then choosing the minimum value over all regions, this would not be practical for large values of $N$.
Therefore we continue with the sampling perspective to develop an efficient sufficient condition for verifying the feasibility of~\eqref{eq:mf_recourse}. 
We propose testing feasibility with the condition
\begin{align}
\vec v_0 + B\vec a_i \in\calP_i ~\text{for}~\vec a_i \in \arg\min_{\calR^k} \|B\vec a+\vec v_0-\itf_i\|^2_2\:.
\end{align}
By checking feasibility for each $i$, we verify a lower bound on the \emph{amount of recourse} available to a user, considering their specific rating history and the allowable actions.

Note that if the control matrix $B$ is full rank, then we can find a point $\vec a_i$ such that $\vec v_0 + B\vec a_i = \vec \itf_i $, meaning that items who are aligned-reachable are also reachable by users.
The rank of $B$ is equal to the rank of $Q_{\Omega_m}$, so as previously observed, item availability implies recourse for any user with control over at least $d$ ratings whose corresponding item factors are linearly independent.

Even users with incomplete control have some level of
recourse.
For the following result, we define $\proj_B$ as the projection matrix onto the subspace spanned by $B$.
Then let $\itf_{B,i}=\proj_B\itf_i$ be the component of the target item factor that lies in the space spanned by the control matrix $B$, and $\vec v_{\perp} = \vec v_0 - \proj_B\vec v_0$ be the component of the anchor point that cannot be affected by user control. 

\begin{result} \label{res:user_aligned_reach}
When $\calR = \R$,  a lower bound on the amount of recourse for a user $u$ is given by the percent of unseen items that satisfy:
\begin{align*}
\|\itf_{B,i}\|_2^2 &+ \itf_i^\top \vec v_{\perp} > {\maxn_{j\neq i}}~ \left(\itf_j^\top \itf_{B,i} + \itf_j^\top \vec v_{\perp} \right)\:.
\end{align*}

\end{result}
Note that this statement mirrors the sufficient condition for items, with modifications relating both to the directions of user control and the anchor point.
In short, user recourse follows from the ability to modify ratings for a set of diverse items, and  
immutable ratings ensure the reachability of some items, potentially at the expense of others. %

\begin{figure}
\center
\includegraphics[width=\fullsize\figwidthbase]{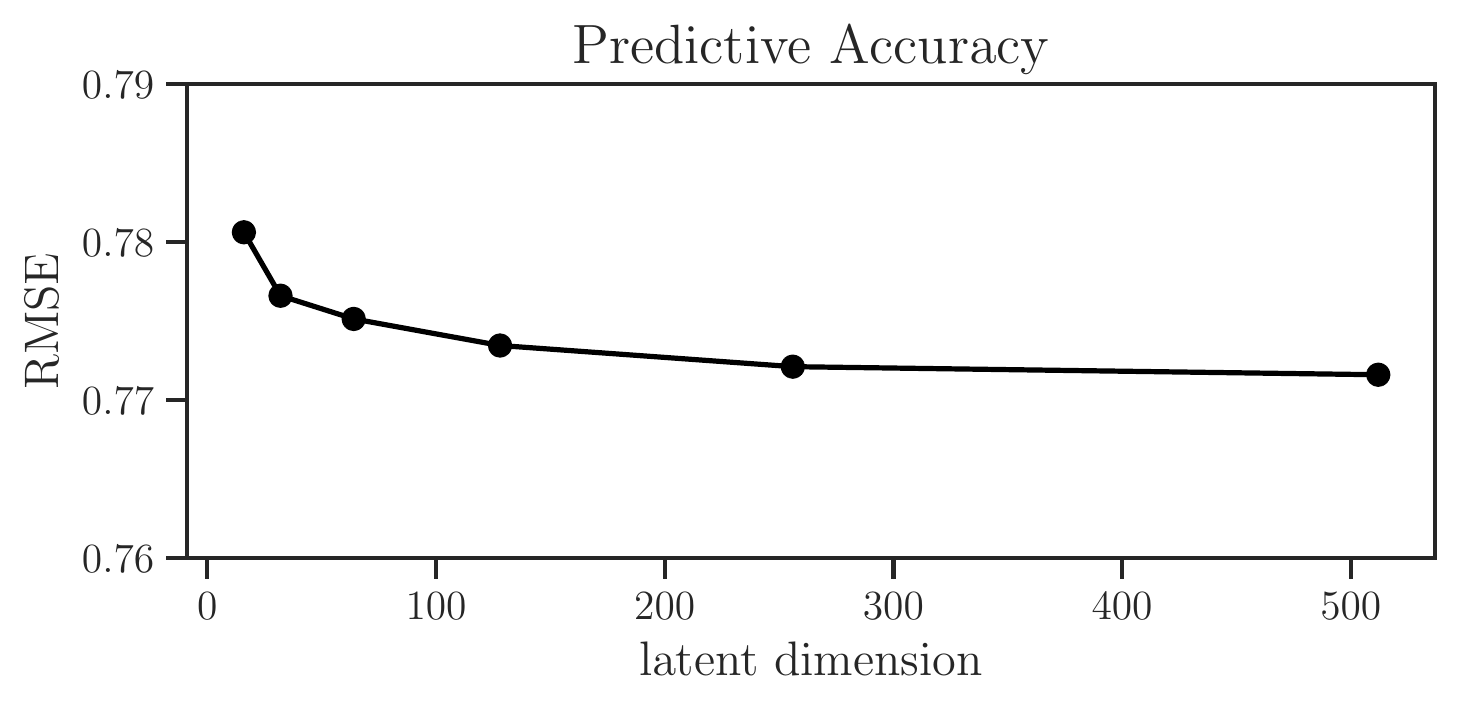}
\caption{The test RMSE of the matrix factorization models on the MovieLens dataset.}\label{fig:model_acc}
\end{figure}
\section{Experimental Demonstration} \label{sec:experiments}

In this section, we demonstrate how our proposed analyses can be used a a tool to audit and interpret characteristics of a matrix factorization model.
We use 
the MovieLens 10M dataset, which comes from an online movie recommender service called MovieLens~\cite{harper2016movielens}.
The dataset\footnote{
\url{http://grouplens.org/datasets/movielens/10m/}}  contains approximately 10 million ratings applied to 10,681 movies by 71,567 users.
The ratings fall between $0$ and $5$ in $0.5$ increments.

We chose this dataset because it is a common benchmark for evaluating rating predictions.
Using the method described by~\citet{rendle2019difficulty} in their recent work on baselines for recommender systems, we train a regularized matrix factorization model.
This is model incorporates item, user, and overall bias terms. 
\ifcompileapp
(Appendix~\ref{sec:app_bias}
\else
(Appendix A
\fi
includes full description of adapting our proposed audits to this model).

We examine models of a variety of latent dimension ranging from $d=16$ to $d=512$. 
The models were trained using the libfm\footnote{\url{http://www.libfm.org/}} library~\cite{rendle:tist2012}.
We use the regression objective and optimize using SGD with regularization parameter $\lambda=0.04$ and step size $0.003$ for $128$ epochs on $90\%$ of the data, verifying accuracy on the remaining $10\%$ with a random global test/train split. 
These methods match those presented by~\citet{rendle2019difficulty} and reproduce their reported accuracies (Figure~\ref{fig:model_acc}).

In 
\ifcompileapp
Appendix~\ref{sec:experiments_app},
\else
Appendix B,
\fi
we present a similar set of experiments on the LastFM dataset.

\begin{figure}
\center
\includegraphics[width=\fullsize\figwidthbase]{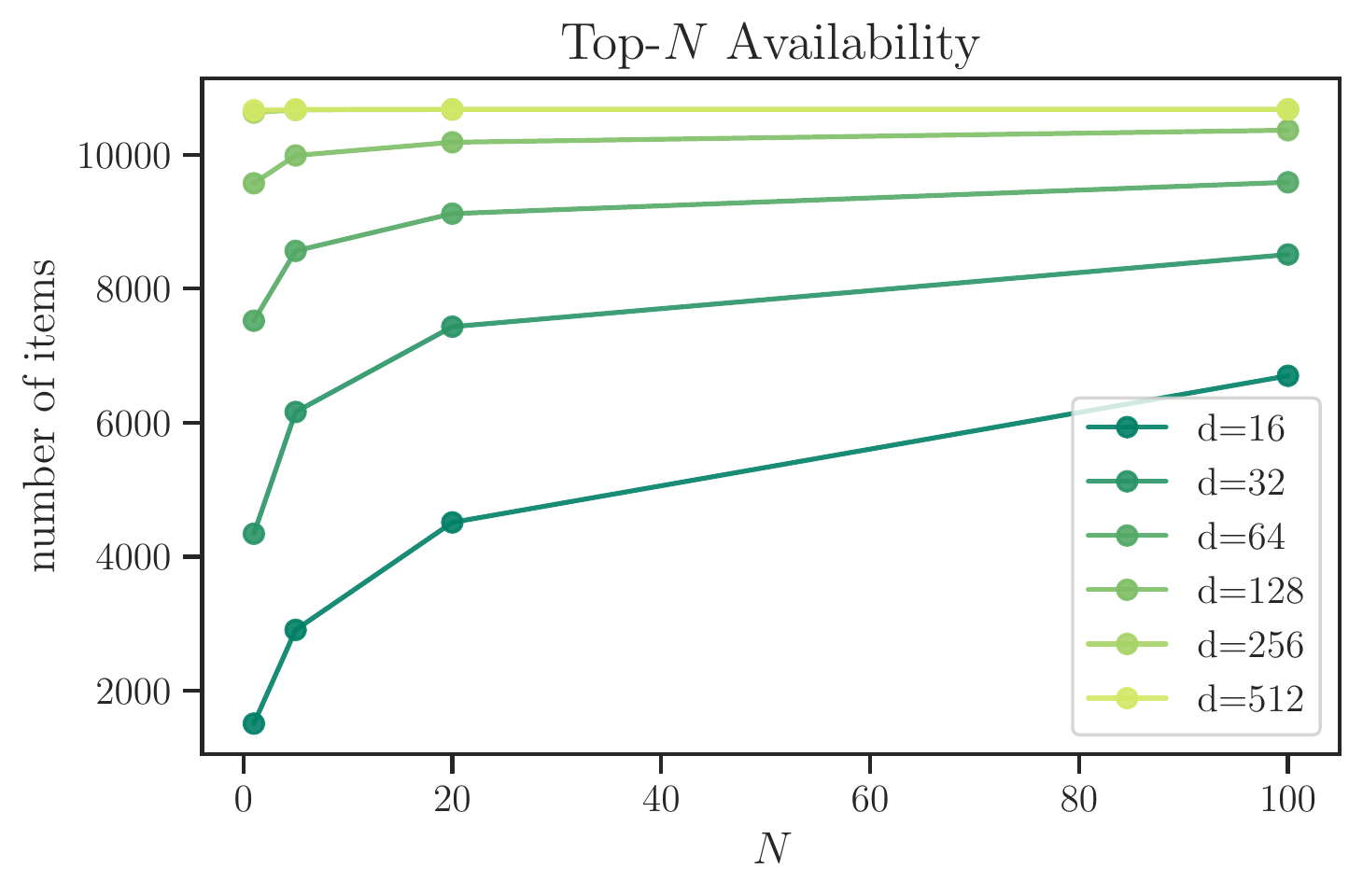}
\caption{Only some of the 10,681 total movies are aligned-reachable, especially for models with smaller complexity and for smaller recommendation set sizes $N$.}\label{fig:item_v_N}
\end{figure}

\begin{figure}
\center
\includegraphics[width=\fullsize\figwidthbase]{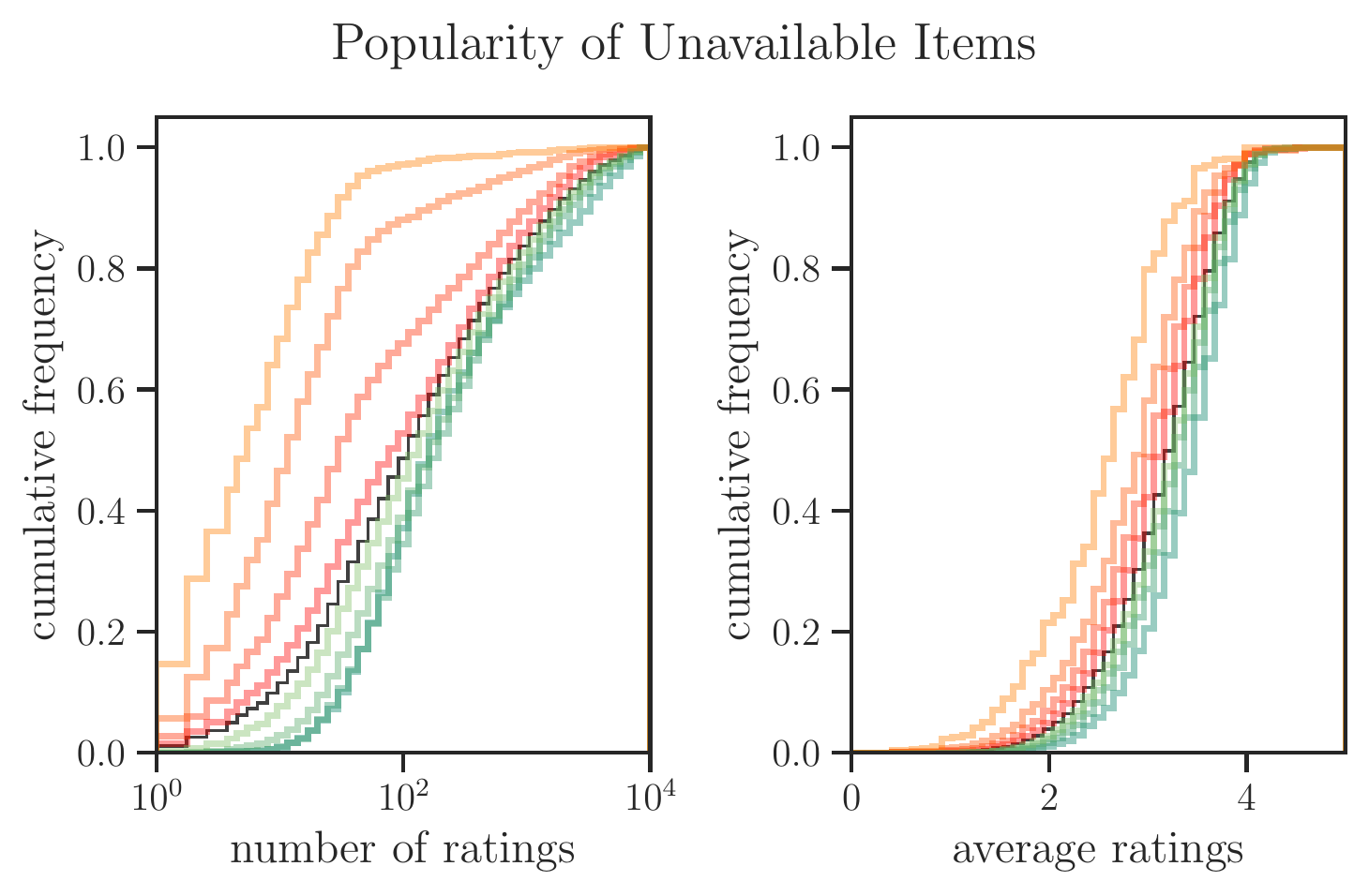}
\includegraphics[width=\fullsize\figwidthbase]{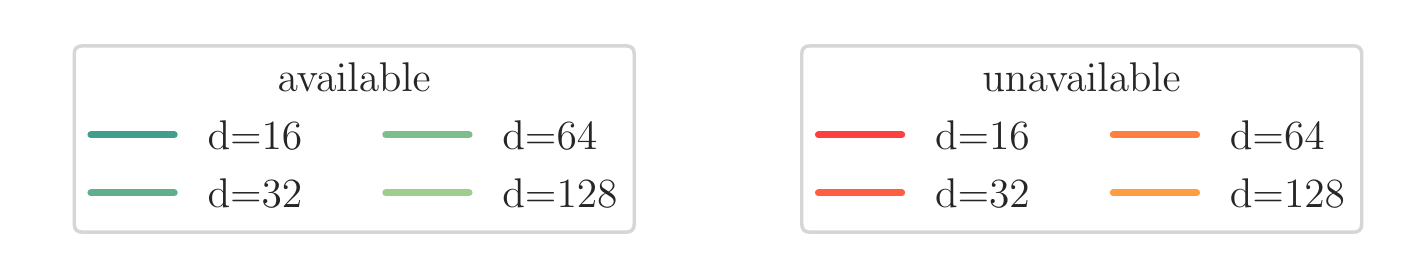}
\caption{
Unavailable items are systematically less popular than available items: they are rated less frequently and have lower average ratings in the training data.
Each curve represents the cumulative density function (CDF) of the popularity measure within the available (green) and unavailable (red) items.
The black line represents the CDF of the combined population.
This trend is true for models of varying complexity. 
}\label{fig:rating_distributions}
\end{figure}

\subsection{Item-Based Audit}
We begin by performing the item-based audit as described in Section~\ref{sec:item-audit}. 
Figure~\ref{fig:item_v_N} displays the total number of aligned-reachable movies for various parameters $N$.
It is immediately clear that all items are baseline-reachable in only the models with the largest latent dimension.
Indeed, we can conclude that the model with $d=16$ has only about $60\%$ availability for users with a history of under 100 movies.
On the other hand, the models with the highest complexity $d\in\{256, 512\}$ have about $99\%$ availability for even the smallest values of $N$. 

We further examine the characteristics of the items that are unavailable compared with those that are available. 
We examine two notions of popularity: total number of ratings and average rating. 
In Figure~\ref{fig:rating_distributions}, we compare the distributions of the available and unavailable items (for $N=5$) in the training set on these measures.
The unavailable items have systematically lower popularity for various latent dimensions.
This observation has implications for the outcome of putting these models in feedback with users.
If unavailable items are never recommended, they will be less likely to be rated, which may exacerbate their unavailability if models are updated online.
We highlight this phenomenon as a potential avenue for future work.

While the difference in popularity is true across all models, it is important to note that there is still overlap in the support of both distributions. 
For a given number of ratings or average rating, some items will be available while others will not, meaning that popularity alone does not determine reachability.

\begin{figure}
\center
\includegraphics[width=\fullsize\figwidthbase]{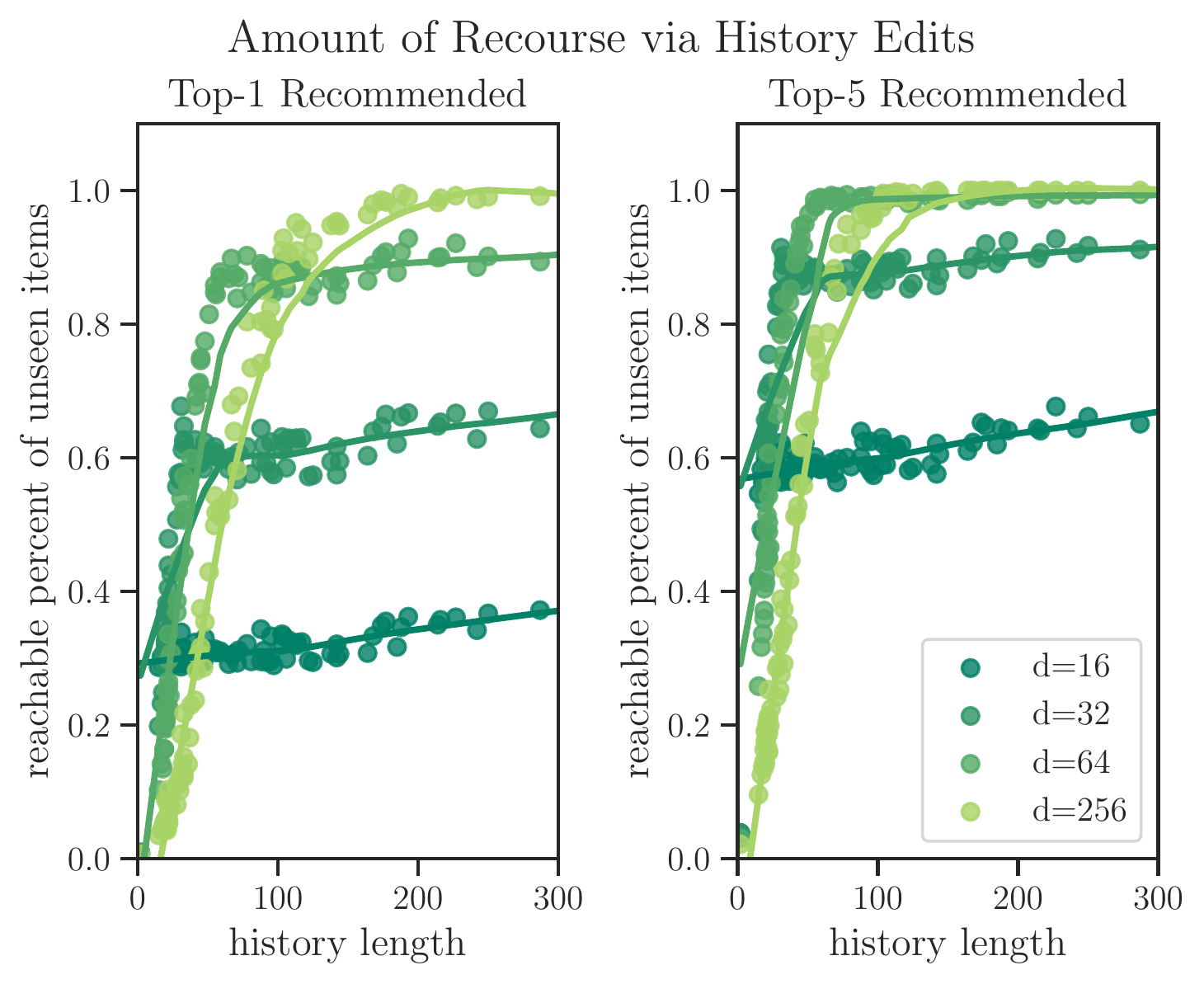}
\caption{The proportion of unseen items reachable by users varies with their history length.  A LOESS regressed curve illustrates the trend. Less complex models are better for shorter history lengths, while more complex models reach higher overall values.}\label{fig:user_hist_reachable}
\end{figure}

\subsection{System Recourse for Users}

Next, we examine the users in this dataset. 
We use the combined testing and training data to determine user ratings $\vec r_u$ and histories $\Omega_u$.
For this section, we examine 100 randomly selected users and only the 1,000 most rated items. 
Sub-selecting items and especially choosing them based on popularity means that these experimental results are an overestimation of the amount of recourse available to users.
Additionally, we allow ratings on the continuous interval $\calR = [0,5]$ rather than enforcing integer constraints, meaning that our results represent the recourse available to users if they were able to precisely rate items on a continuous scale.
Despite these two approximations, several interesting trends on the limits of recourse appear.

We begin with history edits, and compute the amount of recourse that the system provides to users using the sufficient condition in Result~\ref{res:user_aligned_reach}.
Figure~\ref{fig:user_hist_reachable} shows the relationship with the length of user history for several different latent dimensions.
First, note the shape of the curved: a fast increase and then leveling off for each dimension $d$.
For short histories, we see the limiting effect of projection onto the control matrix $\proj_B$. For longer histories, as the rank of $Q_{\Omega_u}$ approaches or exceeds $d$, the baseline item-reachability determines the effect.
The transition between these two regimes differs depending on the latent dimension of the model.
Smaller models reach their maximum quickly, while models of higher complexity provide a larger amount of recourse to users with long histories.
This is an interesting distinction that connects to the idea of ``power users.''

\begin{figure}
\center
\includegraphics[width=\fullsize\figwidthbase]{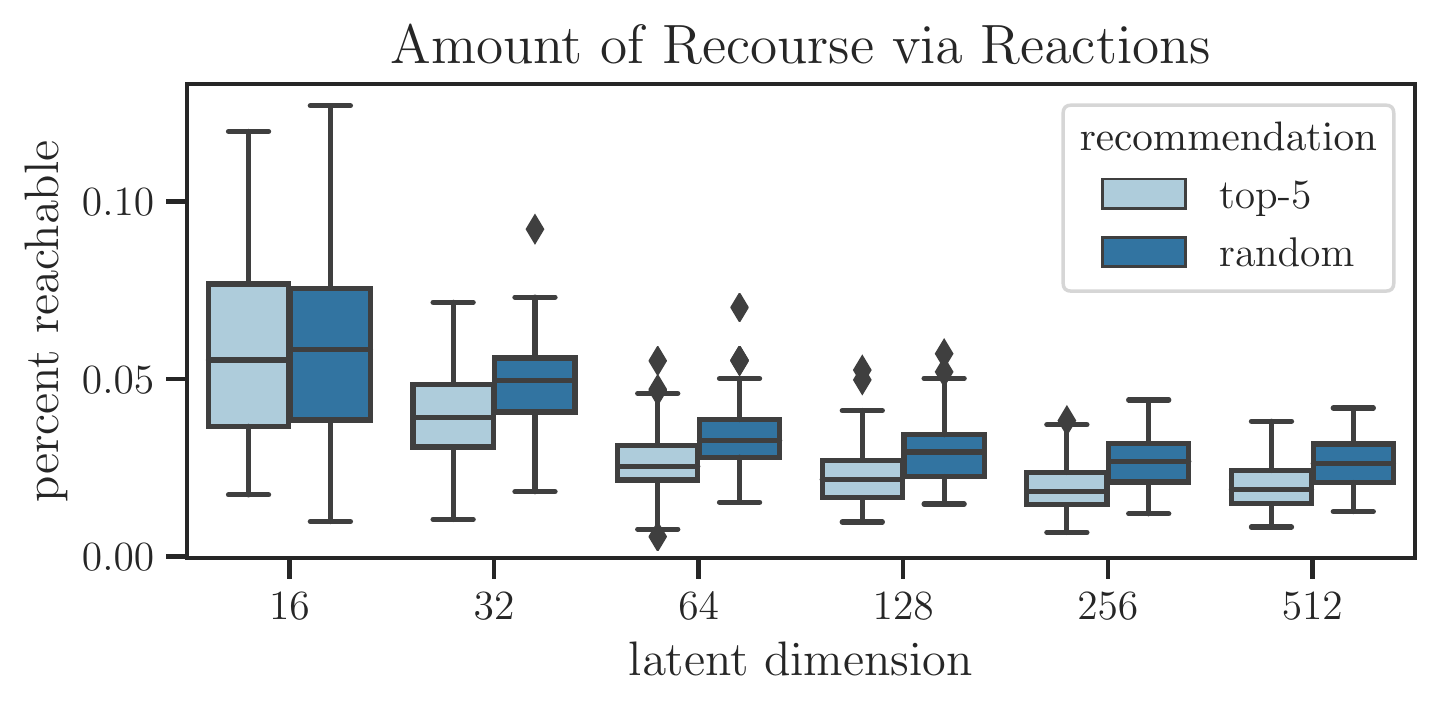}

\includegraphics[width=\fullsize\figwidthbase]{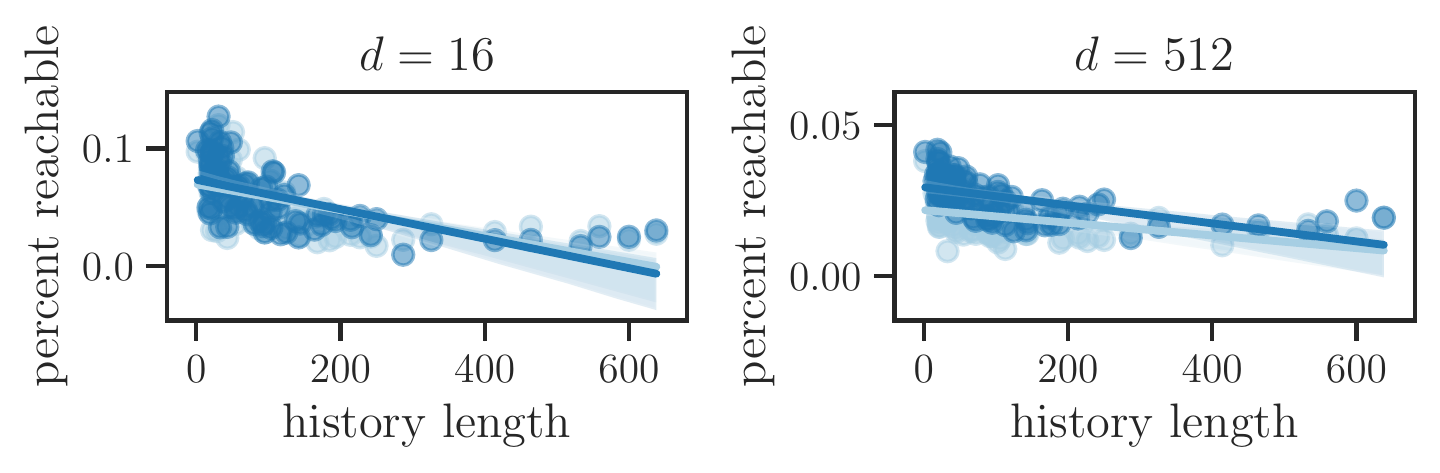}
\caption{When actions are constrained to reaction to a set of items, lower complexity models provide higher reachability. A random set of items provides slightly more recourse to users than if the set is selected based on predicted user preferences. Furthermore, there is a slight trend that users with smaller history lengths have more available recourse. }\label{fig:user_next_reachable}
\end{figure}

Next, we consider reactions, where user
input comes only through reaction to a new set of items while the existing ratings are fixed.
Figure~\ref{fig:user_next_reachable} displays the amount of recourse for two different types of new items:
first, the case that users are shown a completely random set of $5$ unseen items and second, the case that they are shown the $5$ items with the highest predicted ratings.
The top panel displays the amount of recourse provided by each model and each type of recommendation.
There are two important trends.
First, smaller models offer larger amounts of recourse--this is because we are in the regime of few mutable ratings, analogous to the availability of items to users with short histories in the previous figure.
Second, for each model size, the random recommendations provide more recourse than the top-$5$, and though the gap is not large it is consistent.

In the bottom panel of Figure~\ref{fig:user_next_reachable}, we further examine how the length of user history interacts with this model of user behavior. 
For both the smallest and the largest latent dimensions, there is a downwards trend between reachability and history length. 
This does not contradict the trend displayed in Figure~\ref{fig:user_hist_reachable}: in the reactions setting, the rating history manifests as the anchor point $\vec v_0$ rather than additional degrees of freedom in the control matrix $B$ .
It is interesting in light of recent works examining the usefulness of recency bias in recommender systems~\cite{matuszyk2015forgetting}.

Finally, we investigate the difficulty of recourse over all users and a single item.
In this case, we consider top-$1$ recommendations to reduce the computational burden of computing the exact set $\calP_i$.
We pose the cost as the size of the difference between the user input $\vec a$ and the predicted ratings in the $\ell_1$ norm.
Figure~\ref{fig:user_next_reachable_difficulty} shows the difficulty of recourse via reaction for the two types of new items: a completely random set of $20$ unseen items and $20$ items with the highest predicted ratings.
We note two interesting trends.
First, the difficulty of recourse does not increase with model size (even though the amount of recourse is lower).
Second, difficulty is lower for the random set of items than for the top-$20$ items. 
Along with the trend in availability, this suggests a benefit of suggesting items to users based on metrics other than predicted rating.
Future work should more carefully examine methods for constructing recommended sets that trade-off predicted ratings with measures like diversity under the lens of user recourse.

\begin{figure}
\center

\includegraphics[width=\fullsize\figwidthbase]{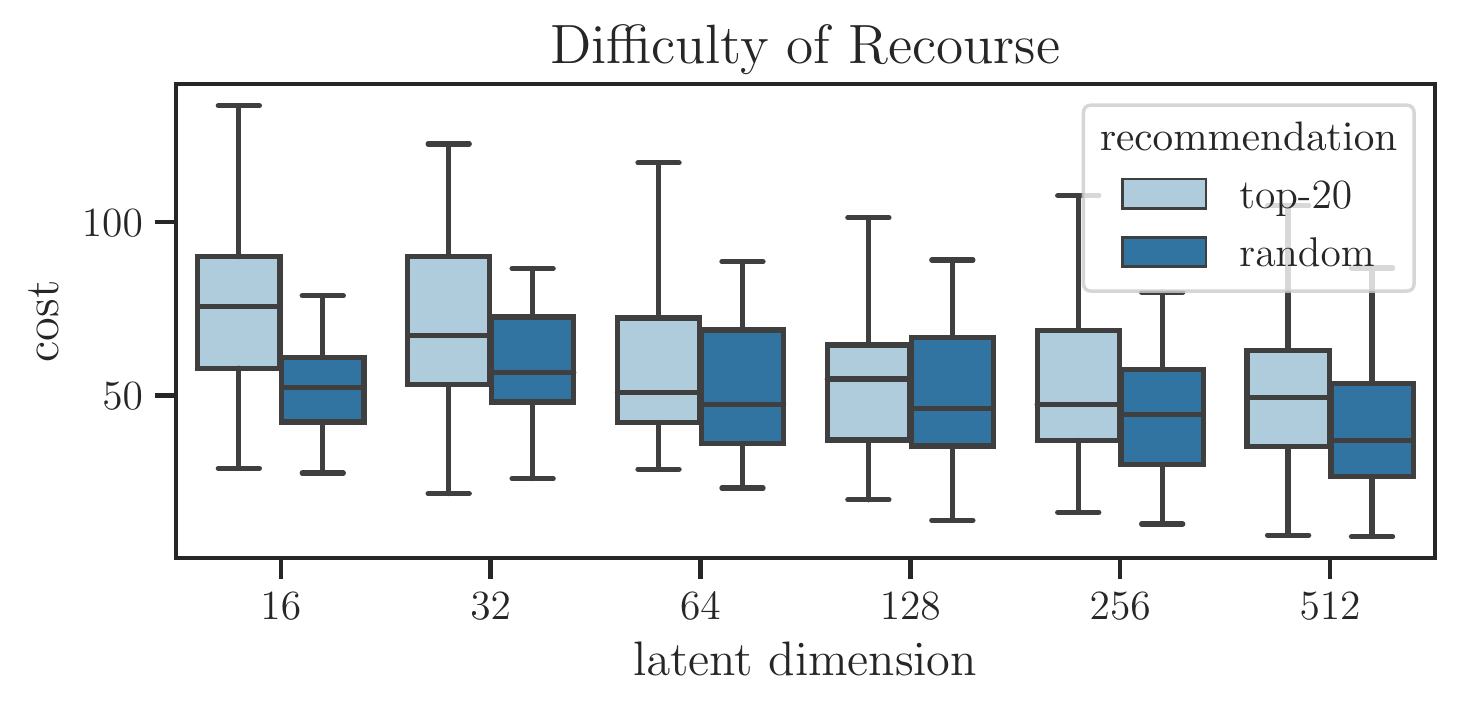}
\caption{
The difficulty of reaching a single item across 100 users for different sets of new times.
The difficulty of recourse does not increase for the larger models, despite the previously observed decrease in availability. 
Furthermore, we note that random items have lower difficulty.
 }\label{fig:user_next_reachable_difficulty}
\end{figure} \section{Discussion}

In this paper, we consider the effects of using predicted user preferences to recommend content, a practice prevalent throughout the Internet today. 
By defining a reachability problem for top-$N$ recommenders, we provide a way to evaluate the impact of using these predictive models to mediate the discovery of content. 
In applying these insights to linear preference models, 
we see several interesting phenomena. 
The first is simple but worth stating: good predictive models, when used to moderate information, can unintentionally make portions of the content library inaccessible to  users. This is illustrated in practice in our study of the MovieLens and LastFM datasets.

To some extent, the differences in the availability of items are
related to their unpopularity within training data. Popularity bias is
a well known phenomenon in which systems fail to personalize~\cite{steck2011item}, and
instead over-recommend the most popular pieces of content. Empirical
work shows connections between popularity bias and undesirable
demographic biases, including the under-recommendation
of female authors~\cite{ekstrand2018exploring}. 
YouTube was long known to have a popularity bias problem (known as the ``Gangnam Style Problem''), until the recommendations began optimizing for predicted ``watch
time'' over ``number of views.''
Their new model has been criticized for its radicalization and political polarization~\cite{zufecki,wsjyoutube}.
The choice of prediction target can have a large effect
on the types of content users can or are likely to discover,
motivating the use of analytic tools like the ones proposed
here to reason about these trade-offs before deployment.

While the reachability criteria proposed in this work form an important basis for reasoning about the availability of content within a recommender system, they do not guarantee less biased behavior on their own.
Many of the audits consider the feasibility of the recourse problem rather than its cost; thus confirming possible outcomes rather than distinguishing probable ones.
Furthermore, the existence of recourse does not fix problems of filter bubbles or toxic content.
Rather, it illuminates limitations inherent in recommender systems for organizing and promoting content.
There is an important distinction between technically providing recourse and the likelihood that people will actually avail themselves of it. 
If the cost function is not commensurate with actual user behavior this analysis may lend an appearance of fairness without substance.

With these limitations in mind, we mention several ways to extend the ideas presented in this work.
On the technical side, there are different models for rating predictions, especially those that incorporate implicit feedback or perform different online update strategies for users.
Not all simple models are linear--for example, subset based recommendations offer greater scrutability and thus user agency by design~\cite{balog2019transparent}.
Further more, top-$N$ recommendation is not the only option.
Post-processing approaches to the recommender policy $\pi$ could work with existing models to modify their reachability properties. 
For example, ~\citet{steck2018calibrated} proposed a method to ensure that the distribution of recommendations over genres remains the same despite model predictions.

One avenue for addressing more generic preference models and recommender policies is to extend the sampling perspective introduced in Section~\ref{sec:sufficient_sampling} to develop a general framework for black-box recommender evaluation.
By sampling with respect to a user transition model, the evaluation could incorporate notions of dynamics and user agency similar to those presented in this work.

Further future work could push the scope of the problem setting to understand the interactions between users and models over time.
Analyzing connections between training data and the resulting reachability properties of the model would to give context to empirical work showing how biases can be reproduces in the way items are recommended~\cite{ekstrand2018exploring,ekstrand2018all}. 
Similarly, directly considering multiple rounds of interactions between users and the recommendation systems would shed light on how these models evolve over time.
This is a path towards understanding phenomena like filter bubbles and polarization.

More broadly, we emphasize the importance of auditing systems with learning-based components in ways that directly consider the models' behavior when put into feedback with humans. 
In the field of formal verification, making guarantees about the behavior of complex dynamical systems over time has a long history.
There are many existing tools~\cite{asarin2000approximate}, though they are generally specialized to the case of physical systems and suffer from the curse of dimensionality.
We accentuate the abundance of opportunity for developing novel approximations and strategies 
for evaluating large scale machine learning systems. 

\section*{Acknowledgements}
Thanks to everyone at Canopy for feedback and support.
SD is supported by an NSF Graduate Research Fellowship under Grant No. DGE 175281.
BR is generously supported in part by ONR awards N00014-17-1-2191, N00014-17-1-2401, and N00014-18-1-2833, the DARPA Assured Autonomy (FA8750-18-C-0101) and Lagrange (W911NF-16-1-0552) programs, and an Amazon AWS AI Research Award.

\bibliographystyle{plainnat}
\bibliography{refs}

\ifcompileapp
\newpage
\appendix

\section{Full Results} \label{sec:app_bias}

Here, we develop our main results in full generality. 
First, we specify the full form of the linear setting.
We consider the preference model with bias terms:
\[\widehat r_{ui} =  \mu + b_i + c_u+\uf_u^\top \itf_i\:.\]
Here, $\vec b \in \R^m$ is a bias on each item, $\vec c \in \R^n$ is a bias for each user, and $\mu$ is the overall bias. 
In this setting the item-regions are now defined as:
\[\calP_i =\{(\itf_i-\itf_j)^\top \uf>b_j - b_i~\text{all but at most $N$ items }j\notin\Omega\}\:.\]
We define $G_i$ to be a $m-|\Omega|\times d$ matrix with rows given by $(\itf_i-\itf_j)$ for $j\notin\Omega$ and $h_i\in\mathbb{R}^{|\Omega^\comp|}$ to be the vector with entries given by $b_j - b_i$ for $j\notin\Omega$.

Lastly, we consider any model that updates only the user models online, and where this update is affine:
\[\uf_u = A\r_u + \vec d\:.\]
As developed in Section~\ref{sec:recourse_and_availability}, 
then a user's representation changes as a result of their actions
\begin{align*}
\uf &= A_{\Omega_0} \r_0 + A_{\Omega_m} \vec a + A\vec d =  \vec v_0 + B \vec a
\end{align*}
where we define $B = A_{\Omega_m}$ and  $\vec v_0 = A_{\Omega_0} \r_0 + A\vec d$.
Similar to before, the reachability problem becomes:
\begin{align}
\begin{split}\label{eq:linear_recourse}
\minimize_{\vec a \in \calR^{|\Omega_m|}} \quad &\mathrm{cost}([\r_0; \vec a]; \r_u)\\
\text{subject to}\quad& G_i (\vec v_0 + B\vec a) > h_i
\end{split}
\end{align}

\subsection{Examples of Linear Preference Models}

In this section, we outline several models that satisfy the form of linear preference model introduced above.
References can be found in chapters 2 and 3 of~\cite{ricci2011introduction}.

\begin{example}[Item-based neighborhood methods.]
In this model 
\[\widehat r_{ui} = \mu + c_u + b_i + \frac{\sum_{j\in\mathcal N_{i}(u)} w_{ij} (r_{uj} - \mu - c_u - b_j) }{ \sum_{j\in\mathcal N_{i}(u)} |w_{ij}| }\]
where $w_{ij}$ measures the similarity between item $i$ and $j$. Regardless of how these weights are defined, this fits into the linear preference model with $\uf_u=\r_u$, $b_i = b_i + \frac{\sum  w_{ij} b_j }{ \sum |w_{ij}| }$ and 
\[(\itf_i)_j\begin{cases} \frac{w_{ij}}{{ \sum_{j\in\mathcal N_{i}(u)} |w_{ij}| }} &j\in\mathcal N_{i}(u)\\0&\text{o.w.} \end{cases}\]
As long as we don't consider an update to $N_{i}(u)$, the update model holds with $A=I$.
\end{example}

\begin{example}[SLIM.]
Here, the model predicts
\[\widehat r_{ui} = \vec w_i^\top \r_u\]
where $\vec w_i$ are sparse row vectors learned via
\begin{align*}
\min_w~&\frac{1}{2}\|R - RW\|_F^2 + \frac{\beta}{2}\|W\|_F^2 + \lambda\|W\|_1\\
\text{s.t.}~&W\geq 0,\mathrm{diag}(W)=0
\end{align*}
Again, the update model holds with $A=I$.
\end{example}

\begin{example}[Matrix Factorization.]
The only modification from the body of the paper is in the update equation,
\begin{align*}
\uf = (Q_\Omega^\top Q_\Omega + \lambda I)^{-1}Q_\Omega^\top (\r_\Omega + \vec b_\Omega + c_u + \mu)\:.
\end{align*}
where $A = WQ_\Omega^\top$ and $\vec d = WQ_\Omega^\top(\vec b_\Omega + c_u + \mu)$.
\end{example}

\subsection{Main results}

We now restate the main results in this more general setting, and provide proofs.

\begin{proposition}[Result~\ref{res:convex_hull} with bias]
A top-$1$ item-region $i$ is empty if and only if there exists normalized weights $w\geq 0$ such that
\begin{align}
\itf_i = \sum_{\substack{j\neq i}}w_j \itf_j, \quad 
b_i \leq \sum_{\substack{j\neq i}}w_j b_j \:.
\end{align} 
The first expression states that the corresponding item factor $\itf_i$ contained within the convex hull of all item factors $j$, and the second requires that the corresponding weights satisfy a bias domination condition.
\end{proposition}

\begin{proof}
The item region is described by $G_i \uf > h_i$. We relate its non-emptiness to the feasibility of the linear program, for an arbitrarily small $\epsilon > 0$,
\begin{align*}
\min~&0^\top \uf\\
\text{s.t.}~&G_i \uf \geq h_i + \epsilon
\end{align*}
The dual of this program is
\begin{align*}
\max~& (h_i + \epsilon)^\top  \lambda \\
\text{s.t.}~&G_i^\top  \lambda = 0,\quad \lambda \geq 0
\end{align*}
The primal problem is feasible if and only if the dual is bounded.
In the dual problem, for any feasible $\lambda$, $c\lambda$ is also feasible for a scalar $c\geq 0$.
Then letting $c\to\infty$, we see that the dual objective is unbounded if and only if there are any feasible $\lambda$ for which $(h_i + \epsilon)^\top  \lambda > 0$.
Therefore, the unboundedness of the dual is equivalent to the existence of a $\lambda$ satisfying:
\begin{align*}
h_i^\top \lambda \geq 0,\quad 
\lambda \geq 0,\quad
G_i^\top  \lambda = 0 \:.
\end{align*} %
The previous statement can be written equivalently as the existence of some $w$ satisfying:
\begin{align*}
w\geq 0,\quad \sum_{\substack{j\neq i}}w_j = 1,\quad
\sum_{\substack{j\neq i}}w_jb_j  \geq b_i, \quad\sum_{\substack{j\neq i}}w_j \itf_j = \itf_i\:. 
\end{align*} 
\end{proof}

\begin{proposition}[Result~\ref{res:user_convex_hull} with bias]
Suppose that $\calR = \R$ and a user has control matrix $B$.
Then the top-$1$ reachability problem for $i$ is feasible if the corresponding multiplied item factor $B^\top \itf_i$ is a vertex of the convex hull of all multiplied item factors $j$. %

Furthermore, for matrix factorization, if $Q_{\Omega_m}$ has rank equal to the latent dimension of the model $d$, then then item availability implies user recourse.
\end{proposition}
\begin{proof}
We follow the argument for the previous result considering the alternate linear region
$G_i B \vec a > h_i - G_i \vec v_0$.
In this case, the linear coefficients are given by $(B^\top\itf_i-B^\top\itf_j)^\top $, while right hand terms are $b_j - b_i - \itf_i^\top \vec v_0+\itf_j^\top \vec v_0$.
Then feasibility is equivalent to the lack of existence normalized weights $w\geq 0$ such that
\begin{align}
B^\top \itf_i = \sum_{\substack{j\neq i}}w_j B^\top \itf_j, \quad 
b_i+\itf_i^\top \vec v_0 \leq \sum_{\substack{j\neq i}} w_j(b_j+\itf_j^\top \vec v_0)\:.
\end{align} 
A sufficient condition is therefore given by $B^\top \itf_i$ not being contained within the convex hull of $\{B^\top \itf_j\}$, so the first statement follows.

The second statement follows because for matrix factorization, $B=WQ_{\Omega_m}^\top$ has the same rank as $Q_{\Omega_m}$ because $W$ is invertible.
If $B$ has rank $d$, then its left inverse exists, so 
$$B^\top \itf_i\in\mathrm{conv}(\{B^\top \itf_j\}) \iff \itf_i\in\mathrm{conv}(\{\itf_j\})$$
\end{proof}

Recall that $\Pi_B$ is the projection matrix onto the subspace spanned by $B$, $\Pi_B = B(B^\top B)^{-1} B$, while $\Pi_B^\perp$ is the projection onto the orthogonal complement, $\Pi_B^\perp = I - B(B^\top B)^{-1} B$.
For this result we define an aligned reachability condition for item $i$,
\begin{align}\label{eq:assumption}
\proj_B \itf_i + \proj_B^\perp \vec v_0  \in\calP_i\:.
\end{align}

\begin{theorem}[Result \ref{res:cost_bound} with Bias]
Suppose that $\calR = \R$, and assume the aligned reachability condition~\eqref{eq:assumption} holds
for all reachable items $i$. 

Let $\uf_u$ indicate the user's latent factor before any actions are taken or the next set of recommendations are added to the user history: 
\[\vec p_u= (Q_{\Omega_u}^\top Q_{\Omega_u} + \lambda I)^{-1}Q_{\Omega_u}^\top (\r_{u, \Omega_u} + \vec b_{\Omega_u} + c_u + \mu) \:.\]
Then we have the bound on the \emph{difficulty} of recourse. 
 In the case of full history edits ,
\[\text{difficulty of recourse for user $u$} \leq \|B^\dagger\| \cdot \frac{1}{|\Omega^\comp|}\sum_{i\in\Omega^\comp} \|\itf_i-\uf_u\| \:.\]
And in the case of reactions,
\[\text{difficulty of recourse for user $u$} \leq \|B^\dagger\| \cdot \frac{1}{|\Omega^\comp|}\sum_{i\in\Omega^\comp} \|\itf_i-(\uf_u+\uf_b)\| \:.\]
where $\uf_b = W Q_{\Omega_m}^\top (\vec b_{\Omega_m} + c_u + \mu)$ represents the effect of the bias of new items on the user factor.
\end{theorem}
\begin{proof}
We begin in the case of history edits. Here, we have that
\[\mathrm{cost}(\vec a; \r_u) = \|\vec a - \r_{u, \Omega_u}\|\]
By assumption, $\vec a = B^\dagger (\itf_i - \vec v_0) + (I-B^\dagger B)\r_{u, \Omega_u}$ is a feasible point, which we select because it is a minimizer of $\vec a \in\arg\min \|B\vec a+\vec v_0-\itf_i\|^2_2 $.
Then we can write
\begin{align*}
\mathrm{cost}(\vec a; \r_u) &\leq \|B^\dagger (\itf_i - \vec v_0) - B^\dagger B\r_{u, \Omega_u}\|\\
&=\|B^\dagger (\itf_i -  W Q_{\Omega_u}^\top (\vec b_{\Omega_u} + c_u + \mu)) - W Q_{\Omega_u}^\top\r_{u, \Omega_u})\|\\
&\leq \|B^\dagger\| \cdot \| \itf_i -  W Q_{\Omega_u}^\top (\r_{u, \Omega_u} + \vec b_{\Omega_u} + c_u + \mu)) \|
\end{align*}
By definition of $\vec p_u$ we have shown the result.

Now we consider the case of new reactions. Here, we have that
\[\mathrm{cost}(\vec a, \r_u) = \|\vec a - \widehat \r_{\Omega_m}\| = \|\vec a - Q_{\Omega_m} (Q_{\Omega_0}^\top Q_{\Omega_0} + \lambda I)^{-1} Q_{\Omega_0}^\top \r_0\|\:.\]
As before, we chose 
$\vec a = B^\dagger (\itf_i - \vec v_0) + (I-B^\dagger B)\widehat \r_{\Omega_m}$,
\begin{align*}
\mathrm{cost}(\vec a; \r_u) \leq \|B^\dagger (\itf_i - \vec v_0-W Q_{\Omega_m}^\top \widehat \r_{\Omega_m}) \|
\end{align*}
Then recall that
\[\widehat \r_{\Omega_m} = Q_{\Omega_m} (Q_{\Omega_0}^\top Q_{\Omega_0} + \lambda I)^{-1} Q_{\Omega_0}^\top (\r_0 + \vec b_{\Omega_0} + c_u + \mu)\]
Furthermore, 
\begin{align*}
\vec v_0 &= W Q_{\Omega_0}^\top \r_0  + W Q_\Omega^\top (\vec b_\Omega + c_u + \mu)\\
&= W Q_{\Omega_0}^\top (\r_0+\vec b_{\Omega_0} + c_u + \mu)  + W Q_{\Omega_m}^\top (\vec b_{\Omega_m} + c_u + \mu)\\
\end{align*}
Then letting
$W_0 = (Q_{\Omega_0}^\top Q_{\Omega_0} + \lambda I)^{-1}$,
\begin{align*}
\vec v_0+W Q_{\Omega_m}^\top \widehat \r_{\Omega_m} = &W (I + Q_{\Omega_m}^\top Q_{\Omega_m} W_0)Q_{\Omega_0}^\top (\r_0 + \vec b_{\Omega_0} + c_u + \mu)\\
&+ W Q_{\Omega_m}^\top (\vec b_{\Omega_m} + c_u + \mu)
\end{align*}
Then we notice that
\begin{align*}
(I + Q_{\Omega_m}^\top Q_{\Omega_m} W_0) &= (W_0^{-1} + Q_{\Omega_m}^\top Q_{\Omega_0})W_0\\
&=(Q_{\Omega_0}^\top Q_{\Omega_0} + \lambda I + Q_{\Omega_m}^\top Q_{\Omega_m})W_0 =W^{-1}W_0
\end{align*}
Therefore, we conclude that
\[\vec v_0+W Q_{\Omega_m}^\top \widehat \r_{\Omega_m}=\uf_{u}+ W Q_{\Omega_m}^\top (\vec b_{\Omega_m} + c_u + \mu)
\]
And we arrive at the bound
\begin{align*}
\mathrm{cost}(\vec a; \r_u) &\leq \|B^\dagger (\itf_i - (\uf_{u} + \uf_b))\|
\end{align*}
\end{proof}

 \begin{proposition}[Result~\ref{res:item_aligned} with bias]
The item-region $\calP_i$ is nonempty if
\begin{align}
\delta_i = \|\itf_i\|_2^2 + b_i - {\maxn_{j\notin \Omega\cup\{i\}}}~ \left(\itf_j^\top \itf_i + b_j\right) > 0\:. \label{eq:aligned_reach_condition_bias}
\end{align} 
The percent of aligned-reachable items lower bounds the percent of baseline-reachable items.
\end{proposition}
\begin{proof}
When the condition is true, $\itf_i \in \calP_i$ so it is nonempty.
Because this is only a sufficient condition, the number of aligned-reachable items is less than or equal to the number of reachable items.
\end{proof}

\begin{proposition}
When $\calR = \R$, the reachability problem~\eqref{eq:general_recourse} is feasible if
\begin{align}
\begin{split}
\label{eq:user_reach_suff_bias}
\|\proj_B\itf_i\|_2^2 &+ \itf_i^\top \proj_B^{\perp}\vec v_0 + b_i
\\&- {\max_{j\neq i}}^{(N)}~ \left(\itf_j^\top 
\proj_B\itf_i + \itf_j^\top \proj_B^\perp\vec v_0 + b_j \right) > 0\:.
\end{split}
\end{align}
\end{proposition}
\begin{proof}
If $\calR=\R$, we have the test point
\begin{align*}
\vec a_i &\in \arg\min \|B\vec a+\vec v_0-\itf_i\|^2_2 \\
&:= B^\dagger(\itf_i-\vec v_0)
\end{align*}
Then verifying its feasibility:
\begin{align*}
\itf_i^\top (\vec v_0 + B\vec a_i)+b_i&=\itf_i^\top (\vec v_0 + BB^\dagger(\itf_i-\vec v_0))+b_i\\
&=\itf_i^\top \proj_B^\perp \vec v_0 + \|\proj_B \itf_i\|_2^2+b_i\\
&> {\max_{j\neq i}}^{(N)}~ \left(\itf_j^\top 
\proj_B\itf_i + \itf_j^\top \proj_B^\perp\vec v_0 +b_j \right)\\
&= {\max_{j\neq i}}^{(N)}~ \left(\itf_j^\top (\vec v_0 + B\vec a_i) +b_j\right)
\end{align*}
where the inequality follows from the property~\eqref{eq:user_reach_suff_bias}, so
we have that $\vec v_0 + B\vec a_i \in\calP_i$.
\end{proof}

\newpage
\section{Experiments with Additional Dataset} \label{sec:experiments_app}

In this section, we present results on an additional dataset
We use 
the LastFM\footnote{\url{Last.fm}} 1K dataset~\cite{Celma:Springer2010}.
The dataset contains records user play counts 992 users with nearly 1 million play counts of songs by 177,023 artists.
We aggregate listens by artist, and remove artists with less than 50 total listens, resulting in a smaller dataset with 23,835 artists, 638,677 datapoints, and the same number of users.
We further transform the number of listens with $\log(1+x)$.
We plot the distribution of ratings in the MovieLens dataset and log-listens in this dataset in Figure~\ref{fig:data-dist}.

\begin{figure}
\center
\includegraphics[width=\fullsize\figwidthbase]{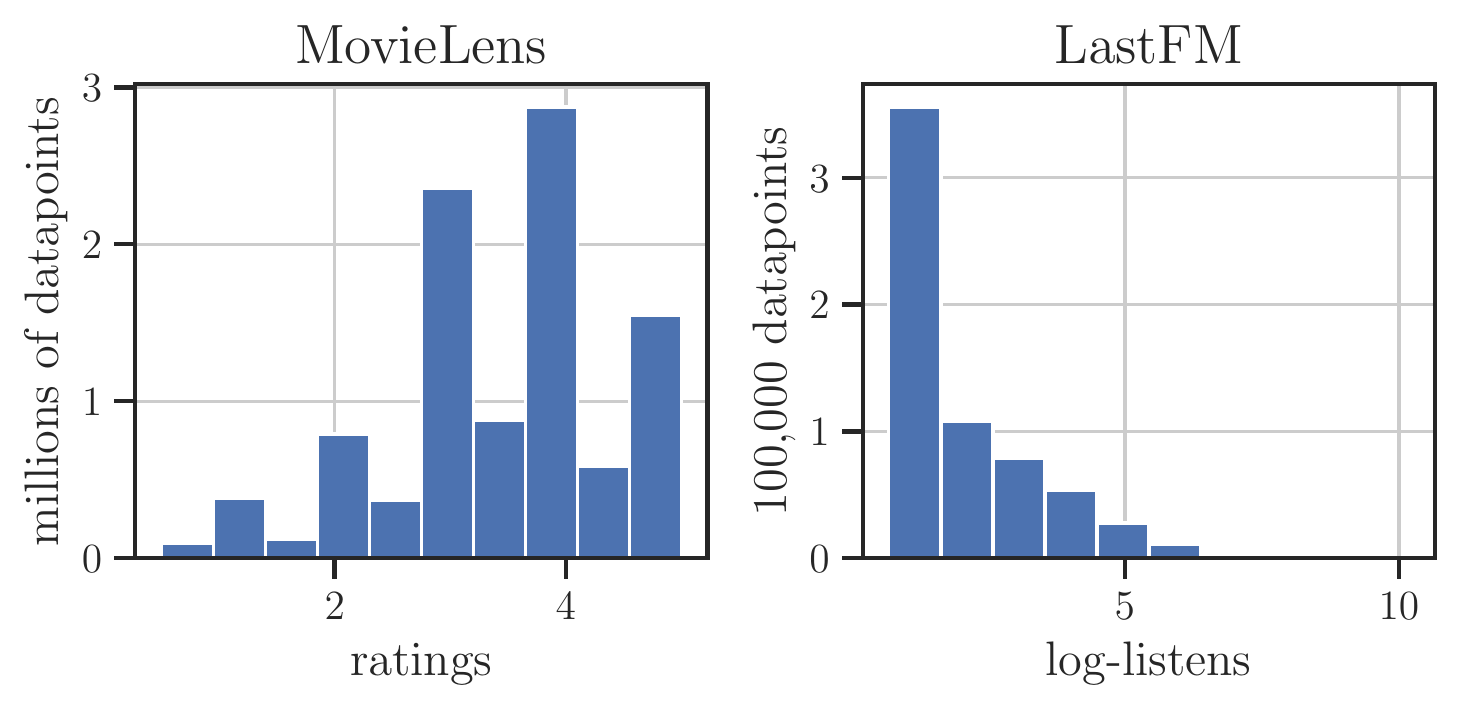}
\caption{Distribution of prediction targets for MovieLens and LastFM datasets (ratings and log-listens, respectively).}
\label{fig:data-dist}
\end{figure}

As with the MovieLens data, we train a regularized matrix factorization model
and examine models of a variety of latent dimension ranging from $d=16$ to $d=512$. 
The models were trained using libfm,
optimized using SGD with regularization parameter $\lambda=0.08$ and step size $0.001$ for $128$ epochs on $90\%$ of the data, verifying accuracy on the remaining $10\%$ with a random global test/train split.

\begin{figure}
\center
\includegraphics[width=\fullsize\figwidthbase]{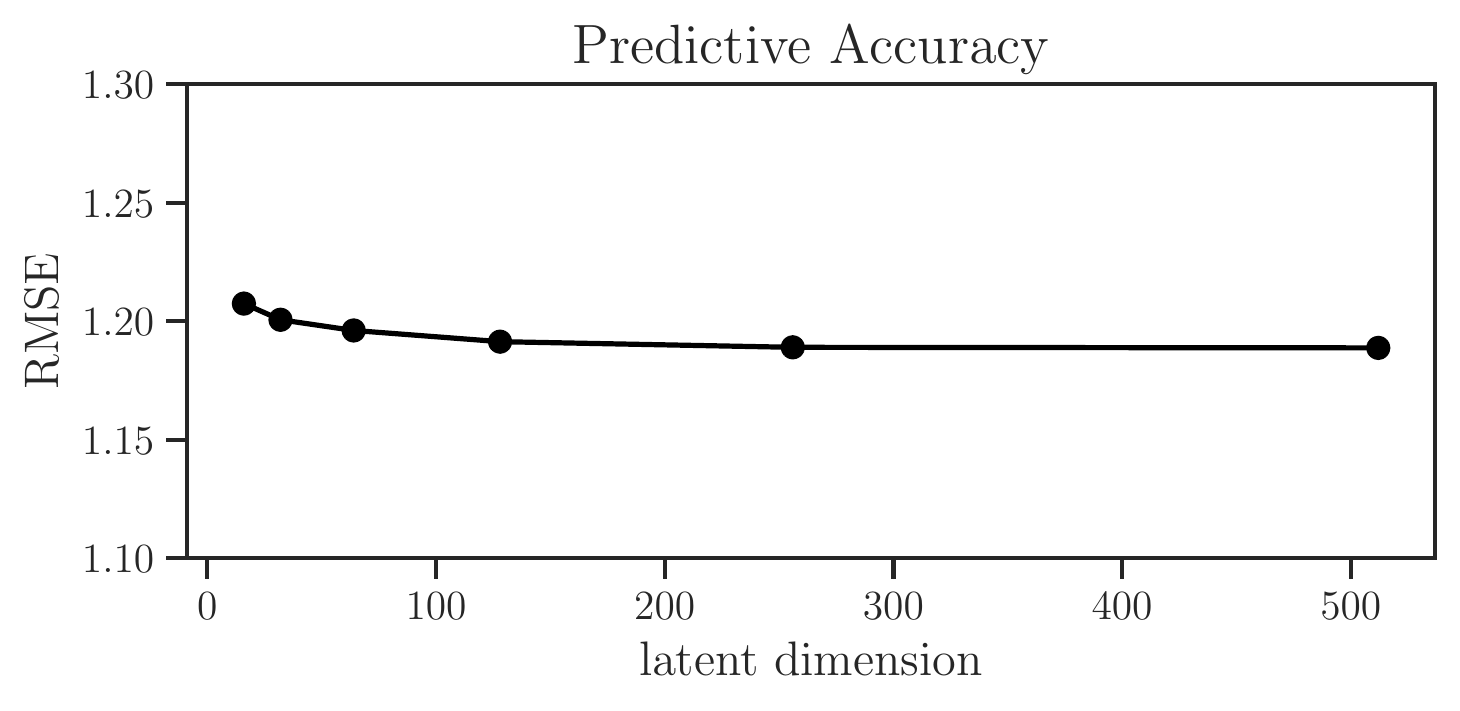}
\caption{The test RMSE of the matrix factorization models on the LastFM dataset.}\label{fig:lf_model_acc}
\end{figure}

Figure~\ref{fig:lf_item_v_N} displays the result of the item-based audit.
Few of the items are baseline-reachabile for models with smaller latent dimension; this trend is more exaggerated than with the MovieLens dataset. 
In Figure~\ref{fig:lf_rating_distributions}, we compare the popularity of the available and unavailable items (for $N=20$) in the training set on these measures.
As before, unavailable items tend have systematically lower popularity, but the trend is less consistent than it was for MovieLens.
This may be due to the more heavily skewed target distribution (Figure~\ref{fig:data-dist}).

\begin{figure}
\center
\includegraphics[width=\fullsize\figwidthbase]{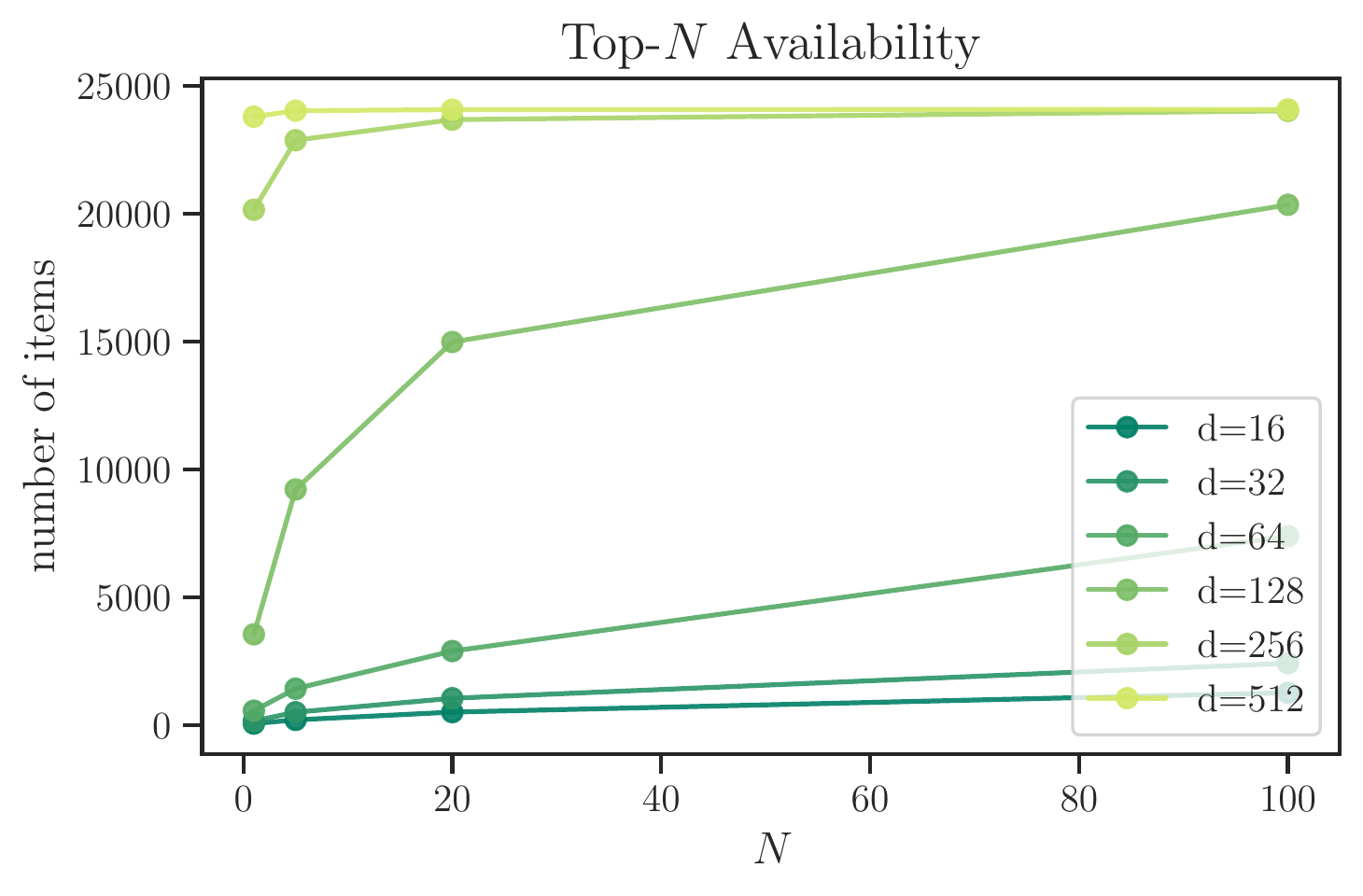}
\caption{Only some of the 23,835 total artists are aligned-reachable, especially for models with smaller complexity and for smaller recommendation set sizes $N$.}\label{fig:lf_item_v_N}
\end{figure}

\begin{figure}
\center
\includegraphics[width=\fullsize\figwidthbase]{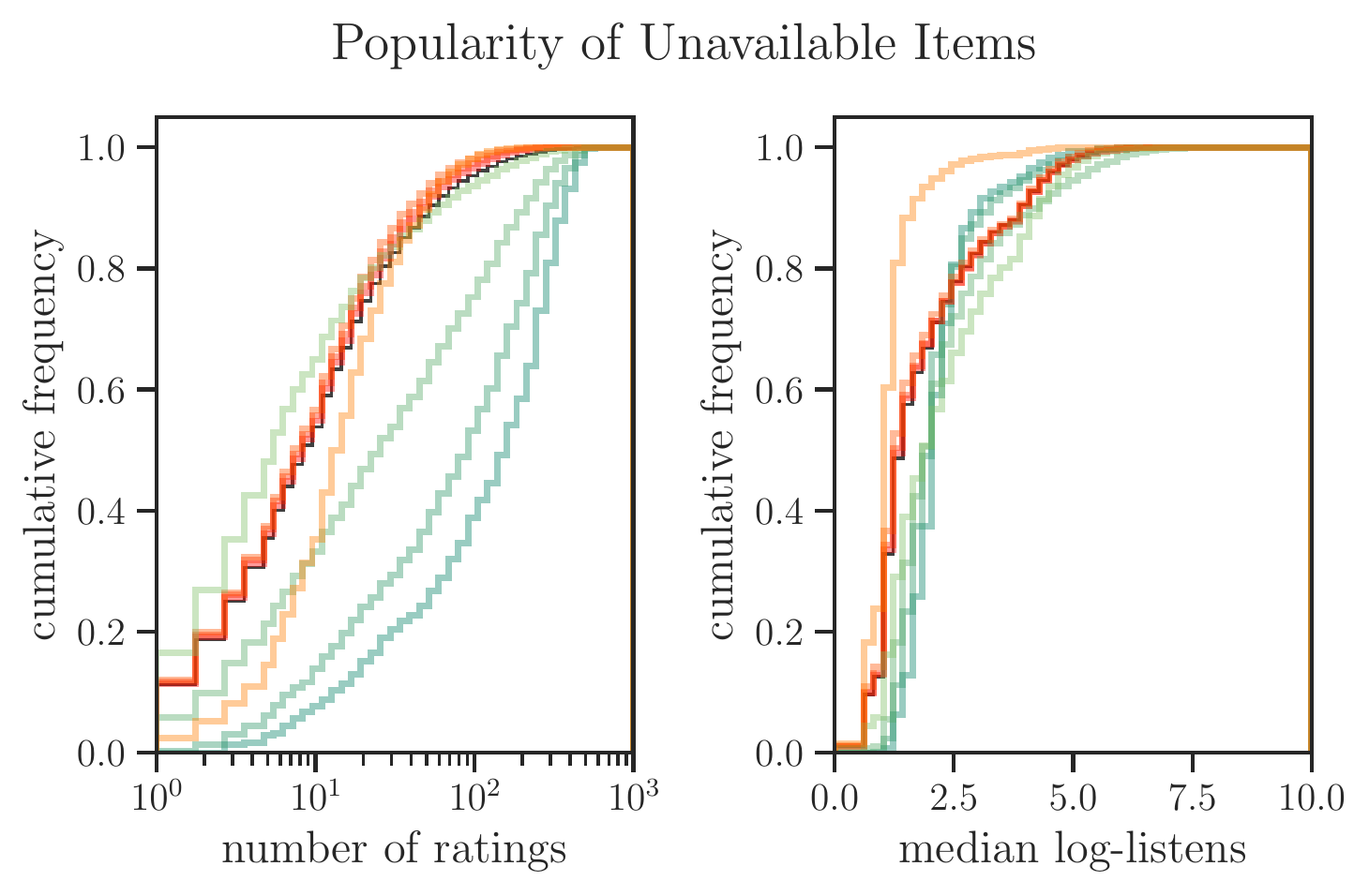}
\includegraphics[width=\fullsize\figwidthbase]{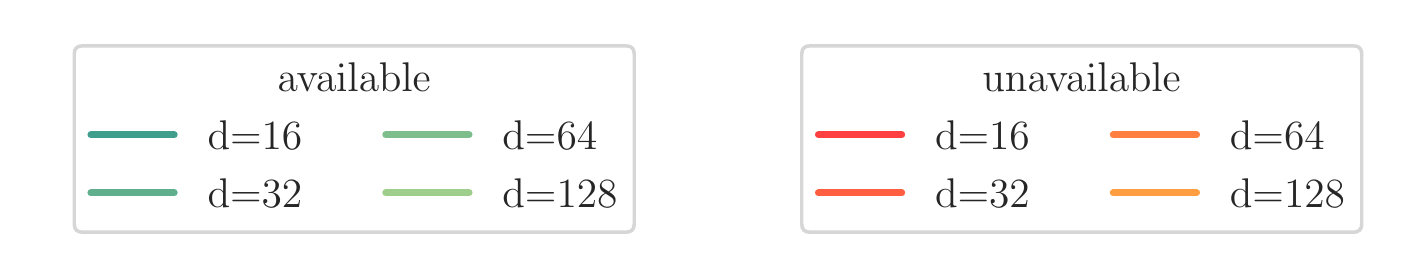}
\caption{
Unavailable items tend to be less popular than available items.
Each curve represents the cumulative density function (CDF) of the popularity measure within the available (green) and unavailable (red) items.
The black line represents the CDF of the combined population.
This trend is true for models of varying complexity. 
}\label{fig:lf_rating_distributions}
\end{figure}

Next, we examine the users in this dataset, combining testing and training data to determine user ratings $\vec r_u$ and histories $\Omega_u$.
For this section, we examine 100 randomly selected users and only the 1,000 most popular items and allow ratings on the continuous interval $\calR = [0,5]$.

Figure~\ref{fig:lf_user_hist_reachable} shows the amount of recourse that the system provides to users via history edits.
We observe is a similar relationship with the length of user history as with the MovieLens data.
Figure~\ref{fig:lf_user_next_reachable} displays the amount of recourse via reactions for random and recommended items.
The previously observed trends are less definitive, but present: smaller models tend to offer more recourse, as do random recommendations.
Furthermore, the bottom panels suggest a negative relationship between recourse and history length, also less definitively than in the MovieLens data.

\begin{figure}
\center
\includegraphics[width=\fullsize\figwidthbase]{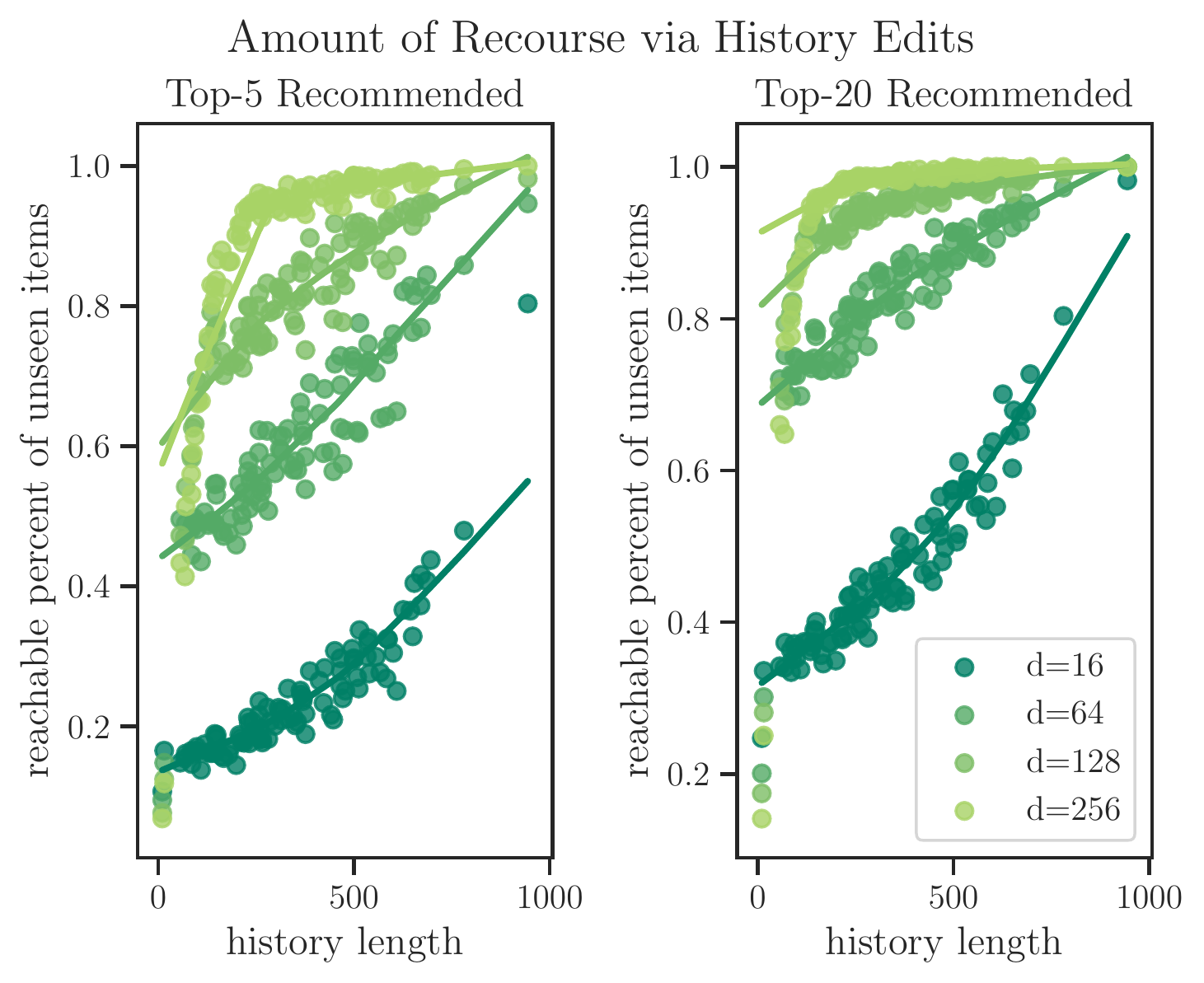}
\caption{The proportion of unseen items reachable by users varies with their history length.  A LOESS regressed curve illustrates the trend.}\label{fig:lf_user_hist_reachable}
\end{figure}

\begin{figure}
\center
\includegraphics[width=\fullsize\figwidthbase]{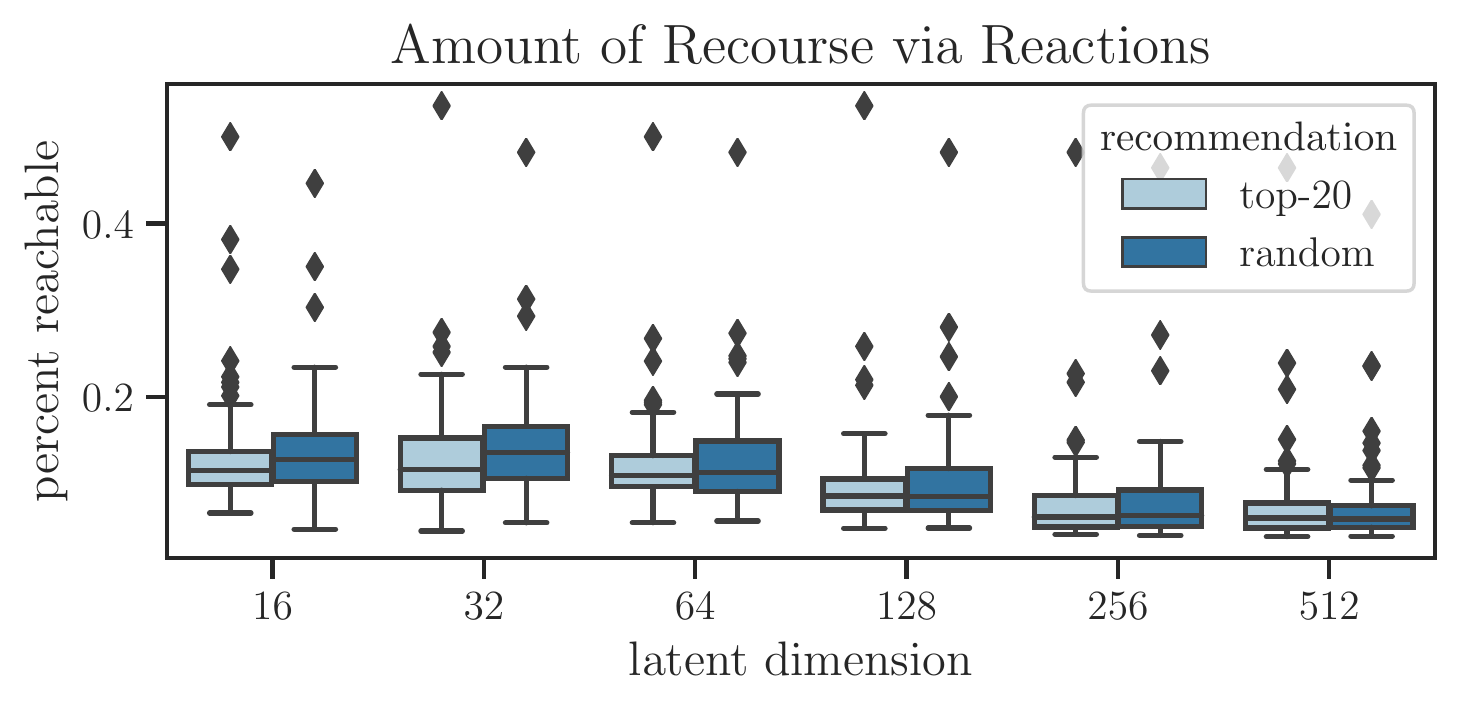}
\includegraphics[width=\fullsize\figwidthbase]{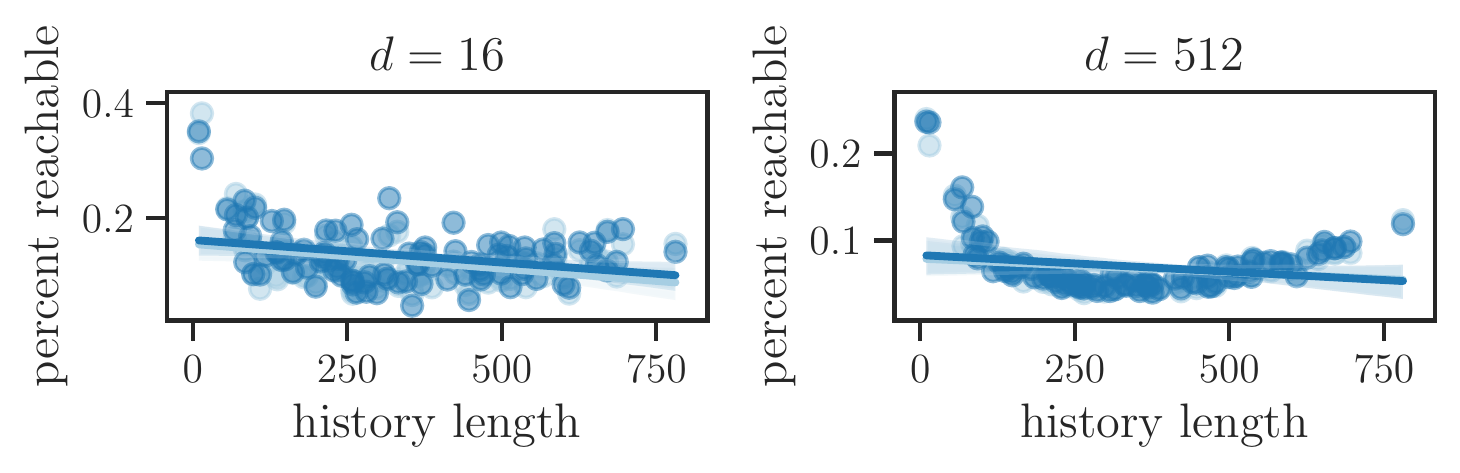}
\caption{When actions are constrained to reaction to a set of items, lower complexity models provide higher reachability. A random set of items provides slightly more recourse to users than if the set is selected based on predicted user preferences. Furthermore, there is a slight trend that users with smaller history lengths have more available recourse. }\label{fig:lf_user_next_reachable}
\end{figure}
 \fi
\end{document}